\documentclass[a4paper,11pt]{article}
\usepackage[cmex10]{amsmath}
\usepackage{amsthm}
\usepackage{amssymb}
\usepackage{algorithmic}
\usepackage{array}
\usepackage{fixltx2e}

\usepackage{stfloats}
\usepackage{color}
\usepackage{url}

\usepackage[vcentering,dvips]{geometry}
\newtheoremstyle{my_style}
{5pt}
{5pt}
{\itshape}
{}
{\bfseries}
{.}
{.5em}
{}

\theoremstyle{my_style}
\newtheorem{theorem}{Theorem}[section]
\newtheorem{lemma}[theorem]{Lemma}
\newtheorem{proposition}[theorem]{Proposition}

\newcommand{\cH}{\mathcal{H}}

\newcommand{\bW}{\boldsymbol{W}}

\newcommand{\bR}{\boldsymbol{R}}
\newcommand{\bS}{\boldsymbol{S}}

\newcommand{\bT}{\boldsymbol{T}}

\newcommand{\bh}{\boldsymbol{h}}

\newcommand{\bc}{\boldsymbol{c}}

\newcommand{\bw}{\boldsymbol{w}}

\newcommand{\bbf}{\boldsymbol{f}}
\newcommand{\bg}{\boldsymbol{g}}
\newcommand{\bZero}{\boldsymbol{0}}

\newcommand{\beq}{\begin{equation}}
\newcommand{\eeq}{\end{equation}}
\newcommand{\beqn}{\begin{eqnarray}}
\newcommand{\eeqn}{\end{eqnarray}}
\newcommand{\beqns}{\begin{eqnarray*}}
\newcommand{\eeqns}{\end{eqnarray*}}

\newcommand{\R}{\mathbb{R}}
\newcommand{\HH}{\mathbb{H}}

\newcommand{\C}{\mathbb{C}}
\newcommand{\A}{\mathbb{A}}

\newcommand{\F}{\mathbb{F}}
\newcommand{\N}{\mathbb{N}}

\newcommand{\frechet}{\textrm{Fr\'{e}chet }}
\newcommand{\fredif}{\textrm{Fr\'{e}chet differentiable }}

\makeatletter
\newcommand{\bdiv}{\mathop{\operator@font div}}
\makeatother

\makeatletter
\newcommand{\diag}{\mathop{\operator@font diag}}
\makeatother

\makeatletter
\newcommand{\conv}{\mathop{\operator@font conv}}
\makeatother

\makeatletter
\newcommand{\sign}{\mathop{\operator@font sign}}
\makeatother \makeatletter

\makeatletter
\newcommand{\proj}{\mathop{\operator@font proj}}
\makeatother \makeatletter

\makeatletter
\newcommand{\spa}{\mathop{\operator@font span}}
\makeatother \makeatletter

\makeatletter
\newcommand{\epi}{\mathop{\operator@font epi}}
\makeatother \makeatletter

\makeatletter
\newcommand{\dom}{\mathop{\operator@font dom}}
\makeatother \makeatletter

\begin{document}
%
\title{\textbf{Wirtinger's Calculus in general Hilbert Spaces}}
%
%
%

\author{Pantelis Bouboulis, Member, IEEE, AMS.
\thanks{P. Bouboulis is with the Department
of Informatics and telecommunications, University of Athens, Greece,
e-mail: (see http://users.uoa.gr/~ldalla/pantelis/).}
}

%
%

\markboth{April 2010}%
{P. Bouboulis, Wirtinger's Calculus in complex Hilbert Spaces}
%



\maketitle

\section{Introduction}
Wirtinger's calculus \cite{Wirti} has become very popular in the signal processing community mainly in the context of
complex adaptive filtering \cite{Picin95, ManGoh, Adali10, Adali08a, Adali08b, MatPaSte, CaGePaVe, Moreno}, as a means of computing, in an elegant way,  gradients of real valued cost functions defined on complex domains ($\C^\nu$). Such functions, obviously, are not holomorphic and therefore the complex derivative cannot be used. Instead, if we consider that the cost function is defined on a Euclidean domain with a double dimensionality ($\R^{2\nu}$), then the real derivatives may be employed. The price of this approach is that the computations become cumbersome and tedious. Wirtinger's calculus provides an alternative equivalent formulation, that is based on simple rules and principles and which bears a great resemblance to the rules of the standard complex derivative. Although this methodology is relatively known in the German speaking countries and has been applied to practical applications \cite{Brandwood, VanDeBos}, only recently has become popular in the signal processing community, mostly due to the works of Picinbono on widely linear estimation filters \cite{Picin95}.

Most Complex Analysis' textbooks deal with \textit{complex analytic} (i.e., \textit{holomorphic}) functions and their properties, which in order to be studied properly a great deal of notions from topology, differential geometry, calculus on manifolds and from other mathematical fields need to be employed. Therefore, most of these materials are accessible only to the specialist. It is only natural that in the scope of this literature, Wirtinger's calculus is usually ignored, since it involves non-holomorphic functions. Nevertheless, in some of these textbooks, the ideas of Wirtinger's calculus are mentioned, although, in most cases, they are presented either superficially, or they are introduced in a completely different set-up (mainly as a bi-product of the Cauchy Riemann conditions). Some of these textbooks are \cite{Remmert, Nehari, MerkHatzi}. However, most of these works are highly specialized and technically abstruse, and therefore not recommended for someone who wants only to understand the concepts of Wirtinger's calculus and use them in his/her field. Moreover, a rigorous and self-consistent presentation of the main ideas of Wirtinger's calculus cannot be found in any of those works. An excellent and highly recommended first attempt to summarize all the related concepts and present them in a unified framework is the introductory text of K. Kreutz-Delgado \cite{Delga}. The aim of the present manuscript is twofold: a) it endeavors to provide a more rigorous presentation of the related material, focusing on aspects that the author finds more insightful and b) it extends the notions of Wirtinger's calculus on general Hilbert spaces (such as Reproducing Hilbert Kernel Spaces).

A common misconception (usually done by beginners in the field) is that Wirtinger's calculus uses an alternative definition of derivatives and therefore results in different gradient rules in minimization problems. We should emphasize that the theoretical foundation of Wirtinger's calculus is the common definition of the real derivative. However, it turns out that when the complex structure is taken into account, the real derivatives may be described using an equivalent and more elegant formulation which bears a surprising resemblance with the complex derivative. Therefore, simple rules may be derived and the computations of the gradients, which may become tedious if the double dimensional space $\R^{2\nu}$ is considered, are simplified.

The manuscript has two main parts. Section \ref{SEC:wirti1}, deals with ordinary Wirtinger's calculus for functions of one complex variable, while in section \ref{SEC:wirti_hilbert} the main notions and results of the extended Wirtinger's Calculus in RKHSs are presented. Throughout the paper, we will denote the set of all integers, real and complex numbers by $\N$, $\R$ and $\C$ respectively. Vector or matrix valued quantities appear in boldfaced symbols.

The present report, has been inspired by the need of the author and its colleagues to understand the underlying theory of Wirtinger's Calculus and to further extend it to include the kernel case. Many parts have been considerably improved after long discussions with prof. S. Theodoridis and L. Dalla.

\section{Wirtinger's Calculus on $\C$}\label{SEC:wirti1}
Consider the function $f:X\subseteq\C\rightarrow\C$, $f(z)=f(x+iy)=u(x,y)+v(x,y)i$, where $u$, and $v$ are real valued functions defined on an open subset $X$ of $\R^2$. Any such function, $f$, may be regarded as defined either on a subset of $\C$ or on a subset of $\R^2$. Furthermore, $f$ may be regarded either as a complex valued function, or as a vector valued function, which takes values in $\R^2$. Therefore, we may equivalently write:
\begin{align*}
f(z)=f(x+iy)=f(x,y)=u(x,y) + iv(x,y) = (u(x,y), v(x,y)).
\end{align*}
The complex derivative of $f$ at $c$, if it exists, is defined as the limit:
\begin{align*}
f'(c)=\lim_{h\rightarrow 0}\frac{f(c+h)-f(c)}{h}.
\end{align*}
This definition, although similar with the typical real derivative of elementary calculus, exploits the complex structure of $X$. More specifically, the division that appears in the definition is based on the complex multiplication, which forces a great deal of structure on $f$. From this simple fact follow all the important strong properties of the complex derivative, which do not have counterparts in the ordinary real calculus. For example, it is well known that if $f'$ exists, then so does $f^{(n)}$, for $n\in\N$.  If $f$ is differentiable at any $z_0\in A$, $f$ is called \textit{holomorphic} in $A$, or \textit{complex analytic} in $A$, in the sense that it can be expanded as a Taylor series, i.e.,
\begin{align}\label{EQ:complex_Taylor}
f(c+h) = \sum_{n=0}^{\infty} \frac{f^{(n)}(c)}{n!} h^n.
\end{align}
The proof of this statement is out of the scope of this manuscript, but it can be found at any complex analysis textbook. The expression ``\textit{$f$ is complex analytic at  $z_0$}'' means that $f$ is complex-analytic at a neighborhood around $z_0$. We will say that $f$ is \textit{real analytic}, when both $u$ and $v$ have a Taylor's series expansion in the real domain.

\subsection{Cauchy-Riemann conditions}\label{SEC:CR_cond}
We begin our study, exploring the relations between the complex derivative and the real derivatives. In the following we will say that $f$ is \textit{differentiable in the complex sense}, if the complex derivative exists, and that $f$ is \textit{differentiable in the real sense}, if both $u$ and $v$ have partial derivatives.

\noindent\rule[1ex]{\linewidth}{1pt}
\begin{proposition}
If the complex derivative of $f$ at a point $c$ (i.e.,$f'(c)$) exists, then $u$ and $v$ are differentiable at the point $(c_1,c_2)$, where $c=c_1+c_2 i$. Furthermore,
\begin{align}\label{EQ:cauchy-riemann}
\frac{\partial u}{\partial x}(c_1,c_2)=\frac{\partial v}{\partial y}(c_1,c_2) \textrm{ and } \frac{\partial u}{\partial y}(c_1,c_2)=-\frac{\partial v}{\partial x}(c_1,c_2).
\end{align}
\end{proposition}
\noindent\rule[1ex]{\linewidth}{1pt}
There are several proves of this proposition, that can be found in any complex analysis textbook. Here we present the two most characteristic ones.\\

\begin{proof}[1st Proof]
Since $f$ is differentiable in the complex sense, the limit:
\begin{align*}
f'(c)=\lim_{h\rightarrow 0}\frac{f(c+h)-f(c)}{h}
\end{align*}
exists. Consider the special case where $h=h_1$ (i.e., $h\rightarrow 0$ on the real axes). Then
\begin{align*}
\frac{f(c+h)-f(c)}{h} &=\frac{u(c_1+h_1, c_2) + iv(c_1+h_1, c_2) - u(c_1,c_2) - iv(c_1,c_2)}{h_1}\\
&= \frac{u(c_1+h_1, c_2) - u(c_1,c_2)}{h_1} + i\frac{v(c_1+h_1, c_2) - v(c_1,c_2)}{h_1}.
\end{align*}
In this case, since the left part of the equation converges to $f(c)$, the real and imaginary parts of  the second half of the equation must also be convergent. Thus, $u$, and $v$ have partial derivatives with respect to $x$ and $f'(c)=\partial u(c)/\partial x + i\partial v(c)/\partial x$. Following the same rationale, if we set $h=ih_2$ (i.e., $h\rightarrow 0$ on the imaginary axes), we take
\begin{align*}
\frac{f(c+h)-f(c)}{h} &=\frac{u(c_1, c_2+h_2) + iv(c_1, c_2+h_2) - u(c_1,c_2) - iv(c_1,c_2)}{ih_2}\\
&= \frac{u(c_1, c_2+h_2) - u(c_1,c_2)}{ih_2} + i\frac{v(c_1, c_2+h2) - v(c_1,c_2)}{ih_2}\\
&= \frac{v(c_1, c_2+h_2) - v(c_1,c_2)}{h_2}  -  i\frac{u(c_1, c_2+h_2) - u(c_1,c_2)}{h_2}.
\end{align*}
The last equation guarantees the existence of the partial derivatives of $u$ and $v$ with respect to $y$. We conclude that $u$ and $v$ have partial derivatives and that
\begin{align*}
f'(c)=\frac{\partial u}{\partial x}(c) + i\frac{\partial v}{\partial x}(c) =  \frac{\partial v}{\partial y}(c) - i\frac{\partial u}{\partial x}(c).
\end{align*}
\end{proof}
\begin{proof}[2nd Proof]
Considering the first order Taylor expansion of $f$ around $c$, we take:
\begin{align}\label{EQ:c_tayl1}
f(c+h)=f(c) + f'(c) h + o(|h|),
\end{align}
where the notation $o$ means that $o(|h|)/|h|\rightarrow 0$, as $|h|\rightarrow 0$. Substituting $f'(c)=A+Bi$ we have:
\begin{align*}
f(c+h) =& f(c) + (A+Bi) (h_1+ih_2) + o(h)\\
=& u(c_1,c_2) + Ah_1 - Bh_2 + \Re[o(h)] + i \left(v(c_1,c_2) + B h_1+ A h_2 + \Im[o(|h|)] \right).
\end{align*}
Therefore,
\begin{align}
u(c_1+h_1,c_2+h_2) =& u(c_1,c_2) + Ah_1 - Bh_2 + \Re[o(|h|)],\label{EQ:r_tayl1}\\
v(c_1+h_1,c_2+h_2) =&   v(c_1,c_2) + B h_1+ A h_2 + \Im[o(|h|)]\label{EQ:r_tayl2}.
\end{align}
Since $o(|h|)/|h|\rightarrow 0$, we also have
\begin{align*}
\Re[o(|(h_1,h_2)|)]/\sqrt{h_1^2+h_2^2}\rightarrow 0 \textrm{ and } \Im[o(|(h_1,h_2)|)]/\sqrt{h_1^2+h_2^2}\rightarrow 0 \textrm{ as } h\rightarrow 0.
\end{align*}
Thus, equations (\ref{EQ:r_tayl1}-\ref{EQ:r_tayl2}) are the first order Taylor expansions of $u$ and $v$ around $(c_1,c_2)$. Hence we deduce that:
\begin{align*}
\frac{\partial u}{\partial x}(c_1,c_2)=A, \frac{\partial u}{\partial y}(c_1,c_2)=-B, \frac{\partial v}{\partial x}(c_1,c_2)=B, \frac{\partial v}{\partial y}(c_1,c_2)=A.
\end{align*}
This completes the proof.
\end{proof}

Equations (\ref{EQ:cauchy-riemann}) are called the \textit{Cauchy Riemann conditions}. It is well known that they provide a necessary and sufficient condition, for a complex function $f$ to be differentiable in the complex sense, providing that $f$ is differentiable in the real sense. This is explored in the following proposition.

\noindent\rule[1ex]{\linewidth}{1pt}
\begin{proposition}\label{PRO:CR2}
If  $f$ is differentiable in the real sense at a point $(c_1,c_2)$ and the Cauchy-Riemann conditions hold:
\begin{align}
\frac{\partial u}{\partial x}(c_1,c_2)=\frac{\partial v}{\partial y}(c_1,c_2) \textrm{ and } \frac{\partial u}{\partial y}(c_1,c_2)=-\frac{\partial v}{\partial x}(c_1,c_2),
\end{align}
then $f$ is differentiable in the complex sense at the point $c=c_1+c_2 i$.
\end{proposition}
\noindent\rule[1ex]{\linewidth}{1pt}

\begin{proof}
Consider the first order Taylor expansions of $u$ and $v$ at $c=c_1 + i c_2 = (c_1,c_2)$:
\begin{align*}
u(c+h) &= u(c) + \frac{\partial u}{\partial x}(c) h_1 + \frac{\partial u}{\partial y}(c) h_2 + o(|h|),\\
v(c+h) &= v(c) + \frac{\partial v}{\partial x}(c) h_1 + \frac{\partial v}{\partial y}(c) h_2 + o(|h|).
\end{align*}
Multiplying the second relation by $i$ and adding it to the first one, we take:
\begin{align*}
f(c+h) &= f(c) + \left(\frac{\partial u}{\partial x}(c) + i\frac{\partial v}{\partial x}(c)\right) h_1 + \left(\frac{\partial u}{\partial y}(c)+i\frac{\partial v}{\partial y}(c)\right) h_2 + o(|h|).
\end{align*}
To simplify the notation we may define
\begin{align*}
\frac{\partial f}{\partial x}(c) = \frac{\partial u}{\partial x}(c) + i\frac{\partial v}{\partial x}(c) \textrm{ and } \frac{\partial f}{\partial y}(c) = \frac{\partial u}{\partial y}(c) + i\frac{\partial v}{\partial y}(c)
\end{align*}
and obtain:
\begin{align*}
f(c+h) &= f(c) + \frac{\partial f}{\partial x}(c)  h_1 + \frac{\partial f}{\partial y}(c) h_2 + o(|h|).
\end{align*}
Nest, we substitute $h_1$ and $h_2$ using the relations $h_1= \frac{h+h^*}{2}$ and $h_2=\frac{h-h^*}{2i}$.
\begin{align}
f(c+h) &= f(c) + \frac{1}{2}\left(\frac{\partial f}{\partial x}(c)  +\frac{1}{i} \frac{\partial f}{\partial y}(c)\right) h +
\frac{1}{2}\left(\frac{\partial f}{\partial x}(c)  -\frac{1}{i} \frac{\partial f}{\partial y}(c)\right) h^* + o(|h|)\nonumber\\
    &=f(c) + \frac{1}{2}\left(\frac{\partial f}{\partial x}(c)  -i \frac{\partial f}{\partial y}(c)\right) h +
\frac{1}{2}\left(\frac{\partial f}{\partial x}(c)  + i \frac{\partial f}{\partial y}(c)\right) h^* + o(|h|).\label{EQ:wirti2}
\end{align}
It will be shown that equation (\ref{EQ:wirti2}) is essential for the development of Wirtinger's calculus. To complete the proof of the proposition we compute the fraction that appears in the definition of the complex derivative:
\begin{align*}
\frac{f(c+h)-f(c)}{h}  &= \frac{1}{2}\left(\frac{\partial f}{\partial x}(c)  -i \frac{\partial f}{\partial y}(c)\right) +
\frac{1}{2}\left(\frac{\partial f}{\partial x}(c)  + i \frac{\partial f}{\partial y}(c)\right) \frac{h^*}{h} + \frac{o(|h|)}{h}
\end{align*}
Recall that for the limit $\lim_{h\rightarrow 0}\lambda\frac{h^*}{h}$ to exist, it is necessary that $\lambda=0$\footnote{To prove it, just set $h=r e^{i\theta}$ and let $r\rightarrow 0$, while keeping $\theta$ constant. Then $\lambda\frac{h^*}{h}=\lambda e^{-2i\theta}\rightarrow 0$, if and only if $\lambda=0$.}. Hence, since $o(|h|)/h\rightarrow 0$ as $h\rightarrow 0$, $f$ is differentiable in the complex sense, iff
\begin{align*}
\frac{\partial f}{\partial x}(c)  + i \frac{\partial f}{\partial y}(c)=0.
\end{align*}
However, according to our definition,
\begin{align*}
\frac{\partial f}{\partial x}(c)  + i \frac{\partial f}{\partial y}(c)= \left(\frac{\partial u}{\partial x}(c) - \frac{\partial v}{\partial y}(c)\right) + i \left(\frac{\partial v}{\partial x}(c) + \frac{\partial u}{\partial y}(c)\right).
\end{align*}
Thus, $f$ is differentiable in the complex sense, iff the Cauchy-Riemann conditions hold. Moreover, in this case:
\begin{align*}
f'(c) =& \frac{1}{2}\left(\frac{\partial f}{\partial x}(c)  -i \frac{\partial f}{\partial y}(c)\right)\\
         =& \frac{1}{2}\left(\frac{\partial u}{\partial x}(c) + i\frac{\partial v}{\partial x}(c)\right)
           - \frac{i}{2}\left(\frac{\partial u}{\partial y}(c) + i\frac{\partial v}{\partial y}(c)\right)\\
         =& \frac{1}{2}\left(\frac{\partial u}{\partial x}(c) + \frac{\partial v}{\partial y}(c)\right)
           + \frac{i}{2}\left(\frac{\partial v}{\partial x}(c) - \frac{\partial u}{\partial y}(c)\right)\\
         =& \frac{\partial u}{\partial x}(c) + i\frac{\partial v}{\partial x}(c)\\
         =& \frac{\partial v}{\partial y}(c) - i\frac{\partial u}{\partial y}(c).
\end{align*}
\end{proof}

More information on holomorphic functions and their relation to harmonic real functions may be found in \cite{Remmert, Nehari, MerkHatzi}.

\subsection{Wirtinger's Derivatives}\label{SEC:wirti_deriv}
In many practical applications we are dealing with functions $f$ that are not differentiable in the complex sense. For example, in minimization problems the cost function is real valued and thus cannot be complex-differentiable at every $x\in X$ (unless it is a constant function\footnote{This statement can be proved using the Cauchy Riemann conditions. For any real valued complex function $f$ defined on a , $v$ vanishes. Therefore the Cauchy-Riemann conditions dictate that $\partial u/\partial x = \partial u/\partial y=0$. Hence, $u$ must be a constant.}). In these cases, our only option is to work with the real derivatives of $u$ and $v$. However, this might make the computations of the gradients cumbersome and tedious. To cope with this problem, we will develop an alternative formulation which, although it is based on the real derivatives, it strongly resembles the notion of the complex derivative. In fact, if $f$ is differentiable in the complex sense, the developed derivatives will coincide with the complex ones. To provide some more insights into the ideas that lie behind the derivation of Wirtinger's Calculus, we present an alternative definition of a complex derivative, which we call the \textit{conjugate-complex derivative} at $c$. We shall say that $f$ is \textit{conjugate-complex differentiable} (or that it is differentiable in the \textit{conjugate-complex sense}) at $c$, if the limit
\begin{align}
f_*'(c)=\lim_{h^*\rightarrow 0} \frac{f(c+h^*) -f(c)}{h} = \lim_{h\rightarrow 0} \frac{f(c+h) -f(c)}{h^*}
\end{align}
exists. Following a procedure similar to the one presented in section \ref{SEC:CR_cond} we may prove the following propositions.

\noindent\rule[1ex]{\linewidth}{1pt}
\begin{proposition}
If the conjugate-complex derivative of $f$ at a point $c$ (i.e.,$f_*'(c)$) exists, then $u$ and $v$ are differentiable at the point $(c_1,c_2)$, where $c=c_1+c_2 i$. Furthermore,
\begin{align}\label{EQ:con_cauchy-riemann}
\frac{\partial u}{\partial x}(c_1,c_2)=-\frac{\partial v}{\partial y}(c_1,c_2) \textrm{ and } \frac{\partial u}{\partial y}(c_1,c_2)=\frac{\partial v}{\partial x}(c_1,c_2).
\end{align}
These are called the \textit{conjugate Cauchy Riemann conditions}.
\end{proposition}

\begin{proposition}
If  $f$ is differentiable in the real sense at a point $(c_1,c_2)$ and the conjugate Cauchy-Riemann conditions hold,
then $f$ is differentiable in the conjugate-complex sense at the point $c=c_1+c_2 i$.
\end{proposition}
\noindent\rule[1ex]{\linewidth}{1pt}
If a function $f$ is differentiable in the conjugate-complex sense, at every point of an open set $A$, we will say that $f$ is \textit{conjugate holomorphic} on $A$. As in the case of the holomorphic functions, one may prove that similar strong results hold for conjugate-holomorphic functions. In particular, it can be shown that if $f_*'(z)$ exists for every $z$ in a neighborhood of $c$, then $f$ has a form of a Taylor series expansion around $c$, i.e.,
\begin{align}\label{EQ:con_complex_Taylor}
f(c+h) = \sum_{n=0}^{\infty} \frac{f_*^{(n)}(c)}{n!} (h^*)^n.
\end{align}
In this case, we will say that $f$ is \textit{conjugate-complex analytic} at $c$. Note, that if $f(z)$ is complex analytic at $c$, then $f^*(z)$ is conjugate-complex analytic at $c$.

It is evident that if neither the Cauchy Riemann conditions, nor the conjugate Cauchy-Riemann conditions are satisfied for a function $f$, then the complex derivatives cannot be exploited and the function cannot be expressed neither in terms of $h$ or $h^*$, as in the case of complex or conjugate-complex differentiable functions.  Nevertheless, if $f$ is differentiable in the real sense (i.e., $u$ and $v$ have partial derivatives), we may still find a form of Taylor's series expansion. Recall, for example, that in the proof of proposition \ref{PRO:CR2}, we concluded, based on the first order Taylor's series expansion of $u$, $v$, that (equation (\ref{EQ:wirti2})):
\begin{align*}
f(c+h) =f(c) + \frac{1}{2}\left(\frac{\partial f}{\partial x}(c)  -i \frac{\partial f}{\partial y}(c)\right) h +
\frac{1}{2}\left(\frac{\partial f}{\partial x}(c)  + i \frac{\partial f}{\partial y}(c)\right) h^* + o(|h|).
\end{align*}
One may notice that in the more general case, where $f$ is real-differentiable, it's Taylor's expansion is casted in terms of both $h$ and $h^*$. This can be generalized for higher order Taylor's expansion formulas following the same rationale. Observe also that, if $f$ is complex or conjugate-complex differentiable, this relation degenerates (due to the Cauchy Riemann conditions) to the respective Taylor's expansion formula (i.e., (\ref{EQ:complex_Taylor}) or \ref{EQ:con_complex_Taylor})). In this context, the following definitions come naturally.

We define the \textit{Wirtinger's derivative} (or \textit{W-derivative} for short) of $f$ at $c$ as follows
\begin{align}\label{EQ:wirti_der}
\frac{\partial f}{\partial z}(c) = \frac{1}{2}\left(\frac{\partial f}{\partial x}(c)  -i \frac{\partial f}{\partial y}(c)\right)
= \frac{1}{2}\left(\frac{\partial u}{\partial x}(c) + \frac{\partial v}{\partial y}(c)\right)
      + \frac{i}{2}\left(\frac{\partial v}{\partial x}(c) - \frac{\partial u}{\partial y}(c)\right).
\end{align}
Consequently, the \textit{conjugate Wirtinger's derivative} (or \textit{CW-derivative} for short) of $f$ at $c$ is defined by:
\begin{align}\label{EQ:conj_wirti_der}
\frac{\partial f}{\partial z^*}(c) = \frac{1}{2}\left(\frac{\partial f}{\partial x}(c)  +i \frac{\partial f}{\partial y}(c)\right)
= \frac{1}{2}\left(\frac{\partial u}{\partial x}(c) - \frac{\partial v}{\partial y}(c)\right)
      + \frac{i}{2}\left(\frac{\partial v}{\partial x}(c) + \frac{\partial u}{\partial y}(c)\right).
\end{align}
Note that both the W-derivative and the CW-derivative exist, if $f$ is differentiable in the real sense. In view of these new definitions, equation (\ref{EQ:wirti2}) may now be recasted as follows
\begin{align}\label{EQ:wirti3}
f(c+h) =f(c) + \frac{\partial f}{\partial z}(c) h +
\frac{\partial f}{\partial z^*}(c) h^* + o(|h|).
\end{align}

At first glance the definitions the W and CW derivatives seem rather obscure. Although, it is evident that they are defined so that that they are consistent with the Taylor's formula (equation (\ref{EQ:wirti2})), their computation seems quite difficult. However, this is not the case. We will show that they may be computed quickly using well-known differentiation rules. First, observe that if $f$ satisfies the Cauchy Riemann conditions then the W-derivative degenerates to the standard complex derivative. The following theorem establishes the fundamental property of W and CW derivatives. Its proof is rather obvious.

\noindent\rule[1ex]{\linewidth}{1pt}
\begin{theorem}\label{THE:wirti}
If $f$ is complex differentiable at $c$, then its W derivative degenerates to the standard complex derivative, while its CW derivative vanishes, i.e.,
\begin{align*}
\frac{\partial f}{\partial z}(c)=f'(c),\quad \frac{\partial f}{\partial z^*}(c)=0.
\end{align*}
Consequently, if $f$ is conjugate-complex differentiable at $c$, then its CW derivative degenerates to the standard conjugate-complex derivative, while its W derivative vanishes, i.e.,
\begin{align*}
\frac{\partial f}{\partial z^*}(c)=f'_*(c),\quad \frac{\partial f}{\partial z}(c)=0.
\end{align*}
\end{theorem}
\noindent\rule[1ex]{\linewidth}{1pt}

In the sequel, we will develop the main differentiation rules of Wirtinger's derivatives. Most of the proofs of the following properties are straightforward. Nevertheless, we present them all for completeness.

\noindent\rule[1ex]{\linewidth}{1pt}
\begin{proposition}\label{PRO:w_rule1}
If $f$ is  differentiable in the real sense at $c$, then
\begin{align}
\left(\frac{\partial f}{\partial z}(c)\right)^* &= \frac{\partial f^*}{\partial z^*}(c).\label{EQ:w_rule1}
\end{align}
\end{proposition}
\noindent\rule[1ex]{\linewidth}{1pt}

\begin{proof}
\begin{align*}
\left(\frac{\partial f}{\partial z}(c)\right)^* &=  \frac{1}{2}\left(\frac{\partial u}{\partial x}(c) + \frac{\partial v}{\partial y}(c)\right)
     - \frac{i}{2}\left(\frac{\partial v}{\partial x}(c) - \frac{\partial u}{\partial y}(c)\right)\\
&=  \frac{1}{2}\left(\frac{\partial u}{\partial x}(c) - \frac{\partial (-v)}{\partial y}(c)\right)
     + \frac{i}{2}\left(\frac{\partial (-v)}{\partial x}(c) + \frac{\partial u}{\partial y}(c)\right)\\
&=\frac{\partial f^*}{\partial z^*}(c).
\end{align*}
\end{proof}

\noindent\rule[1ex]{\linewidth}{1pt}
\begin{proposition}\label{PRO:w_rule2}
If $f$ is  differentiable in the real sense at $c$, then
\begin{align}
\left(\frac{\partial f}{\partial z^*}(c)\right)^* &= \frac{\partial f^*}{\partial z}(c).\label{EQ:w_rule2}
\end{align}
\end{proposition}
\noindent\rule[1ex]{\linewidth}{1pt}
\begin{proof}
\begin{align*}
\left(\frac{\partial f}{\partial z^*}(c)\right)^* &=   \frac{1}{2}\left(\frac{\partial u}{\partial x}(c) - \frac{\partial v}{\partial y}(c)\right)
      - \frac{i}{2}\left(\frac{\partial v}{\partial x}(c) + \frac{\partial u}{\partial y}(c)\right)\\
&=   \frac{1}{2}\left(\frac{\partial u}{\partial x}(c) + \frac{\partial (-v)}{\partial y}(c)\right)
      + \frac{i}{2}\left(\frac{\partial (-v)}{\partial x}(c) - \frac{\partial u}{\partial y}(c)\right)\\
&=\frac{\partial f^*}{\partial z}(c).
\end{align*}
\end{proof}

\noindent\rule[1ex]{\linewidth}{1pt}
\begin{proposition}[Linearity]\label{PRO:w_linearity}
If $f$, $g$ are  differentiable in the real sense at $c$ and $\alpha, \beta\in\C$, then
\begin{align}
\frac{\partial (\alpha f + \beta g)}{\partial z}(c) &= \alpha\frac{\partial f}{\partial z}(c) + \beta\frac{\partial g}{\partial z}(c),\\
\frac{\partial (\alpha f + \beta g)}{\partial z^*}(c) &= \alpha\frac{\partial f}{\partial z^*}(c) + \beta\frac{\partial g}{\partial z^*}(c)
\end{align}
\end{proposition}
\noindent\rule[1ex]{\linewidth}{1pt}

\begin{proof}
Let $f(z)=f(x,y)=u_{f}(x,y) + i v_{f}(x,y)$, $g(z)=g(x,y)=u_{g}(x,y) + i v_{g}(x,y)$ be two complex functions and $\alpha, \beta\in\C$, such that $\alpha=\alpha_1 + i\alpha_2$, $\beta=\beta_1 + i\beta_2$. Then
\begin{align*}
 r(z) =& \alpha f(z) + \beta g(z) = (\alpha_1+ i\alpha_2)(u_{f}(x,y) + iv_{f}(x,y)) + (\beta_1+ i\beta_2)(u_{g}(x,y) + iv_{g}(x,y))\\
=& \left(\alpha_1 u_{f}(x,y) - \alpha_2 v_{f}(x,y) + \beta_1 u_{g}(x,y) -\beta_2 v_{g}(x,y)\right)\\
& + i\left(\alpha_1 v_{f}(x,y) +\alpha_2 u_{f}(x,y) + \beta_1 v_{g}(x,y) +\beta_2 u_{g}(x,y) \right).
\end{align*}
Thus, the W-derivative of $ r$ will be given by:
\begin{align*}
\frac{\partial r}{\partial z}(c) =& \frac{1}{2}\left(\frac{\partial u_{ r}}{\partial x}(c) + \frac{\partial v_{ r}}{\partial y}(c)\right)
      + \frac{i}{2}\left(\frac{\partial v_{ r}}{\partial x}(c) - \frac{\partial u_{ r}}{\partial y}(c)\right)\\
   =& \frac{1}{2}\left( \alpha_1 \frac{\partial u_{f}}{\partial x}(c) -\alpha_2 \frac{\partial v_{f}}{\partial x}(c) + \beta_1 \frac{\partial u_{g}}{\partial x}(c) -\beta_2 \frac{\partial v_{g}}{\partial x}(c)
+\alpha_1 \frac{\partial v_{f}}{\partial y}(c)  + \alpha_2 \frac{\partial u_{f}}{\partial y}(c)
+\beta_1 \frac{\partial v_{g}}{\partial y}(c) +\beta_2 \frac{\partial u_{g}}{\partial y}(c) \right)\\
&+\frac{i}{2}\left( \alpha_1 \frac{\partial v_{f}}{\partial x}(c) +\alpha_2 \frac{\partial u_{f}}{\partial x}(c) + \beta_1 \frac{\partial v_{g}}{\partial x}(c) +\beta_2 \frac{\partial u_{g}}{\partial x}(c)
  -\alpha_1 \frac{\partial u_{f}}{\partial y}(c)  + \alpha_2 \frac{\partial v_{f}}{\partial y}(c)
   -\beta_1 \frac{\partial u_{g}}{\partial y}(c) +\beta_2 \frac{\partial v_{g}}{\partial y}(c) \right)\\
=& \frac{1}{2}(\alpha_1+i\alpha_2) \frac{\partial u_{f}}{\partial x}(c)  + \frac{i}{2}(\alpha_1+i\alpha_2) \frac{\partial v_{f}}{\partial x}(c)
+ \frac{1}{2}(\beta_1+i\beta_2) \frac{\partial u_{g}}{\partial x}(c) + \frac{i}{2}(\beta_1+i\beta_2) \frac{\partial v_{g}}{\partial x}(c)\\
  & + \frac{1}{2}(\alpha_1+i\alpha_2) \frac{\partial v_{f}}{\partial y}(c) - \frac{i}{2}(\alpha_1+i\alpha_2) \frac{\partial u_{f}}{\partial y}(c)
  +\frac{1}{2}(\beta_1+i\beta_2) \frac{\partial v_{g}}{\partial y}(c) - \frac{i}{2}(\beta_1+i\beta_2) \frac{\partial u_{g}}{\partial y}(c)\\
=& \alpha\left(\frac{1}{2}\left(\frac{\partial u_{f}}{\partial x}(c) + \frac{\partial v_{f}}{\partial y}(c)\right) +  \frac{i}{2}\left(\frac{\partial v_{f}}{\partial x}(c) - \frac{\partial u_{f}}{\partial y}(c)\right)  \right)\\
&+\beta\left(\frac{1}{2}\left(\frac{\partial u_{g}}{\partial x}(c) + \frac{\partial v_{g}}{\partial y}(c)\right) +  \frac{i}{2}\left(\frac{\partial v_{g}}{\partial x}(c) - \frac{\partial u_{g}}{\partial y}(c)\right)  \right)\\
=& \alpha \frac{\partial f}{\partial z}(c) + \beta \frac{\partial g}{\partial z}(c).
\end{align*}
On the other hand, in view of Propositions \ref{PRO:w_rule2} and \ref{PRO:w_rule1} and the linearity property of the W-derivative, the CW-derivative of $ r$ at $c$ will be given by:
\begin{align*}
\frac{\partial r}{\partial z^*}(c) =& \frac{\partial(\alpha f+\beta g)}{\partial z^*}(c) = \left(\frac{\partial(\alpha f+\beta g)^*}{\partial z}(c)\right)^*\\
    =& \left(\frac{\partial(\alpha^*f^*+\beta^*g^*)}{\partial z}(c)\right)^* = \left(\alpha^*\frac{\partial f^*}{\partial z}(c) + \beta^*\frac{\partial g^*}{\partial z}(c) \right)^*\\
    =& \alpha \left(\frac{\partial f^*}{\partial z}(c)\right)^* + \beta \left(\frac{\partial g^*}{\partial z}(c)\right)^*
    = \alpha \frac{\partial f}{\partial z^*}(c) + \beta \frac{\partial g}{\partial z^*}(c).
\end{align*}
\end{proof}

\noindent\rule[1ex]{\linewidth}{1pt}
\begin{proposition}[Product Rule]\label{PRO:w_product}
If $f$, $g$ are  differentiable in the real sense at $c$, then
\begin{align}
\frac{\partial (f \cdot g)}{\partial z}(c) &= \frac{\partial f}{\partial z}(c)g(c) + f(c)\frac{\partial g}{\partial z}(c),\\
\frac{\partial (f \cdot g)}{\partial z^*}(c) &= \frac{\partial f}{\partial z^*}(c)g(c) + f(c)\frac{\partial g}{\partial z^*}(c).
\end{align}
\end{proposition}
\noindent\rule[1ex]{\linewidth}{1pt}

\begin{proof}
Let $f(z)=f(x,y)=u_{f}(x,y) + i v_{f}(x,y)$, $g(z)=g(x,y)=u_{g}(x,y) + i v_{g}(x,y)$ be two complex functions differentiable at $c$. Consider the complex function $ r$ defined as $ r(z)=f(z)g(z)$. Then
\begin{align*}
 r(z)=(u_{f}(z)+i v_{g}(z))(u_{g}(z) + i v_{g}(z))=(u_{f}u_{g}-v_{f}v_{g}) + i(u_{f}v_{g}+v_{f}u_{g}).
\end{align*}
Hence the W-derivative of $ r$ at $c$ is given by:
\begin{align*}
\frac{\partial  r}{\partial z}(c) =& \frac{1}{2}\left(\frac{\partial u_{ r}}{\partial x}(c) + \frac{\partial v_{ r}}{\partial y}(c) \right) + \frac{i}{2}\left(\frac{\partial v_{ r}}{\partial x}(c) - \frac{\partial u_{ r}}{\partial y}(c) \right)\\
=& \frac{1}{2}\left(\frac{\partial (u_{f}u_{g}-v_{f}v_{g})}{\partial x}(c) + \frac{\partial (u_{f}v_{g}+v_{f}u_{g})}{\partial y}(c) \right) + \frac{i}{2}\left(\frac{\partial (u_{f}v_{g}+v_{f}u_{g})}{\partial x}(c) - \frac{\partial (u_{f}u_{g}-v_{f}v_{g})}{\partial y}(c) \right)\\
=& \frac{1}{2}\left(\frac{\partial (u_{f}u_{g})}{\partial x}(c) - \frac{\partial (v_{f}v_{g})}{\partial x}(c)
+ \frac{\partial (u_{f}v_{g})}{\partial y}(c) + \frac{\partial (v_{f}u_{g})}{\partial y}(c) \right) \\
&+ \frac{i}{2}\left(\frac{\partial (u_{f}v_{g})}{\partial x}(c) + \frac{\partial (v_{f}u_{g})}{\partial x}(c)
- \frac{\partial (u_{f}u_{g})}{\partial y}(c) +  \frac{\partial (v_{f}v_{g})}{\partial y}(c) \right)
\end{align*}
\begin{align*}
=& \frac{1}{2}\left(\frac{\partial u_{f}}{\partial x}(c)u_{g}(c) + \frac{\partial u_{g}}{\partial x}(c)u_{f}(c)
- \frac{\partial v_{f}}{\partial x}(c)v_{g}(c) - \frac{\partial v_{g}}{\partial x}(c)v_{f}(c) \right.\\
&\left. + \frac{\partial u_{f}}{\partial y}(c)v_{g}(c) + \frac{\partial v_{g}}{\partial y}(c)u_{f}(c)
+ \frac{\partial v_{f}}{\partial y}(c)u_{g}(c) + \frac{\partial u_{g}}{\partial y}(c)v_{f}(c)   \right)  \\
&+ \frac{i}{2}\left(\frac{\partial u_{f}}{\partial x}(c)v_{g}(c) + \frac{\partial v_{g}}{\partial x}(c)u_{f}(c)
+ \frac{\partial v_{f}}{\partial x}(c)u_{g}(c) + \frac{\partial u_{g}}{\partial x}(c)v_{f}(c)\right.\\
&\left. - \frac{\partial u_{f}}{\partial y}(c)u_{g}(c)  - \frac{\partial u_{g}}{\partial y}(c) u_{f}(c)
+  \frac{\partial v_{f}}{\partial y}(c)v_{g}(c)   +  \frac{\partial v_{g}}{\partial y}(c)v_{f}(c) \right).
\end{align*}
After factorization we obtain:
\begin{align*}
\frac{\partial  r}{\partial z}(c)
=& u_{g}(c)\left(\frac{1}{2}\left(\frac{\partial u_{f}}{\partial x}(c) +\frac{\partial v_{f}}{\partial y}(c)\right) + \frac{i}{2}\left(\frac{\partial v_{f}}{\partial x}(c) -\frac{\partial u_{f}}{\partial y}(c)\right)\right)\\
&+ v_{g}(c)\left(\frac{1}{2}\left(-\frac{\partial v_{f}}{\partial x}(c) +\frac{\partial u_{f}}{\partial y}(c)\right) + \frac{i}{2}\left(\frac{\partial u_{f}}{\partial x}(c) +\frac{\partial v_{f}}{\partial y}(c)\right)\right)\\
&+ u_{f}(c)\left(\frac{1}{2}\left(\frac{\partial u_{g}}{\partial x}(c) +\frac{\partial v_{g}}{\partial y}(c)\right) + \frac{i}{2}\left(\frac{\partial v_{g}}{\partial x}(c) -\frac{\partial u_{g}}{\partial y}(c)\right)\right)\\
&+ v_{f}(c)\left(\frac{1}{2}\left(-\frac{\partial v_{g}}{\partial x}(c) +\frac{\partial u_{g}}{\partial y}(c)\right) + \frac{i}{2}\left(\frac{\partial u_{g}}{\partial x}(c) +\frac{\partial v_{g}}{\partial y}(c)\right)\right).
\end{align*}
Applying the simple rule $1/i=-i$, we take:
\begin{align*}
\frac{\partial  r}{\partial z}(c)
=& u_{g}(c)\left(\frac{1}{2}\left(\frac{\partial u_{f}}{\partial x}(c) +\frac{\partial v_{f}}{\partial y}(c)\right) + \frac{i}{2}\left(\frac{\partial v_{f}}{\partial x}(c) -\frac{\partial u_{f}}{\partial y}(c)\right)\right)\\
&+ i v_{g}(c)\left(\frac{1}{2}\left(\frac{\partial u_{f}}{\partial x}(c) +\frac{\partial v_{f}}{\partial y}(c)\right) + \frac{i}{2}\left(\frac{\partial v_{f}}{\partial x}(c) - \frac{\partial u_{f}}{\partial y}(c)\right)\right)\\
&+ u_{f}(c)\left(\frac{1}{2}\left(\frac{\partial u_{g}}{\partial x}(c) +\frac{\partial v_{g}}{\partial y}(c)\right) + \frac{i}{2}\left(\frac{\partial v_{g}}{\partial x}(c) -\frac{\partial u_{g}}{\partial y}(c)\right)\right)\\
&+ i v_{f}(c)\left(\frac{1}{2}\left(\frac{\partial u_{g}}{\partial x}(c) +\frac{\partial v_{g}}{\partial y}(c)\right) + \frac{i}{2}\left(\frac{\partial v_{g}}{\partial x}(c) -\frac{\partial u_{g}}{\partial y}(c)\right)\right)\\
=& (u_{g}(c) + i v_{g})\frac{\partial f}{\partial z}(c) + (u_{f}(c) + i v_{f})\frac{\partial g}{\partial z}(c),
\end{align*}
which gives the result.

The product rule of the CW-derivative follows from the product rule of the W-derivative and Propositions \ref{PRO:w_rule1}, \ref{PRO:w_rule2} as follows:
\begin{align*}
\frac{\partial (fg)}{\partial z^*}(c) &= \left(\frac{\partial (fg)^*}{\partial z}(c)\right)^*
=\left(\frac{\partial f^*}{\partial z}(c)g^*(c) + \frac{\partial g^*}{\partial z}(c)f^*(c)\right)^*\\
&= \frac{\partial f}{\partial z^*}(c)g(c) + \frac{\partial g}{\partial z^*}(c)f(c).
\end{align*}
\end{proof}

\begin{lemma}[Reciprocal Rule]\label{PRO:w_reciprocal}
If $f$ is  differentiable in the real sense at $c$ and $f(c)\not=0$, then
\begin{align}
\frac{\partial(\frac{1}{f})}{\partial z}(c) &= - \frac{\frac{\partial f}{\partial z}(c)}{f^2(c)},\\
\frac{\partial(\frac{1}{f})}{\partial z^*}(c) &= - \frac{\frac{\partial f}{\partial z^*}(c)}{f^2(c)}.
\end{align}
\end{lemma}

\begin{proof}
Let $f(z)=u_{f}(x,y) + i v_{f}(x,y)$ be a complex function,  differentiable in the real sense at $c$, such that $f(c)\not=0$. Consider the function $ r(z)=1/f(z)$. Then
\begin{align*}
 r(z)= \frac{u_{f}(z)}{u^2_{f}(z) + v^2_{f}(z)}  -   i\frac{v_{f}(z)}{u^2_{f}(z) + v^2_{f}(z)}.
\end{align*}
For the partial derivatives of $u_{ r}$, $v_{ r}$ we have:
\begin{align*}
\frac{\partial u_{ r}}{\partial x}(c) =& \frac{\frac{\partial u_{f}}{\partial x}(c)(u^2_{f}(c) + v^2_{f}(c)) -2u^2_{f}(c)\frac{\partial u_{f}}{\partial x}(c) - 2u_{f}(c)v_{f}(c)\frac{\partial v_{f}}{\partial x}(c)} {(u^2_{f}(c) + v^2_{f}(c))^2},\\
\frac{\partial u_{ r}}{\partial y}(c) =& \frac{\frac{\partial u_{f}}{\partial y}(c)(u^2_{f}(c) + v^2_{f}(c)) -2u^2_{f}(c)\frac{\partial u_{f}}{\partial y}(c) - 2u_{f}(c)v_{f}(c)\frac{\partial v_{f}}{\partial y}(c)} {(u^2_{f}(c) + v^2_{f}(c))^2},\\
\frac{\partial v_{ r}}{\partial x}(c) =& - \frac{\frac{\partial v_{f}}{\partial x}(c)(u^2_{f}(c) + v^2_{f}(c)) -2v^2_{f}(c)\frac{\partial v_{f}}{\partial x}(c) - 2u_{f}(c)v_{f}(c)\frac{\partial u_{f}}{\partial x}(c)} {(u^2_{f}(c) + v^2_{f}(c))^2},\\
\frac{\partial v_{ r}}{\partial y}(c) =& - \frac{\frac{\partial v_{f}}{\partial y}(c)(u^2_{f}(c) + v^2_{f}(c)) -2v^2_{f}(c)\frac{\partial v_{f}}{\partial y}(c) - 2u_{f}(c)v_{f}(c)\frac{\partial u_{f}}{\partial y}(c)} {(u^2_{f}(c) + v^2_{f}(c))^2}.
\end{align*}
Therefore,
\begin{align*}
\frac{\partial  r}{\partial z}(c) = & \frac{1}{2}\left(\frac{\partial u_{ r}}{\partial x}(c) + \frac{\partial v_{ r}}{\partial y}(c)\right) + \frac{i}{2}\left( \frac{\partial v_{ r}}{\partial x}(c) - \frac{\partial u_{ r}}{\partial y}(c) \right)\\
= & \frac{1}{2(u^2_{f}(c) + v^2_{f}(c))^2}\left( \frac{\partial u_{f}}{\partial x}(c) \left(-u^2_{f}(c) + v^2_{f}(c) + 2iu_{f}(c)v_{f}(c) \right) +
 \frac{\partial v_{f}}{\partial x}(c) \left(-2u_{f}(c)v_{f}(c) - iu^2_{f}+i v^2_{f}(c)\right)   \right.\\
 & \left.  + \frac{\partial v_{f}}{\partial y}(c) \left(-u^2_{f}(c)+v^2_{f}(c)+2i u_{f}(c)v_{f}(c)\right)
 + \frac{\partial u_{f}}{\partial y}(c) \left(2 u_{f}(c)v_{f}(c) + iu^2_{f}(c) - iv^2_{f}(c)\right)    \right)\\
= &  \frac{u^2_{f}(c) - v^2_{f}(c) - 2iu_{f}(c)v_{f}(c)}{2(u^2_{f}(c) + v^2_{f}(c))^2}
 \left( -  \left( \frac{\partial u_{f}}{\partial x}(c) + \frac{\partial v_{f}}{\partial y}(c)\right) - i \left( \frac{\partial v_{f}}{\partial x}(c) - \frac{\partial u_{f}}{\partial y}(c)  \right) \right)\\
= & - \frac{\frac{\partial f}{\partial z}(c)}{f^2(c)}.
\end{align*}

To prove the corresponding rule of the CW-derivative we apply the reciprocal rule of the W-derivative as well as Propositions \ref{PRO:w_rule1}, \ref{PRO:w_rule2}:
\begin{align*}
\frac{\partial(\frac{1}{f})}{\partial z^*}(c) &= \left(\frac{\partial(\frac{1}{f^*})}{\partial z}(c)\right)^*
= \left( - \frac{\frac{\partial f^*}{\partial z}(c)}{\left(f^*(c)\right)^2}\right)^*
=  - \frac{\frac{\partial f}{\partial z^*}(c)}{\left(f(c)\right)^2}.
\end{align*}
\end{proof}

\noindent\rule[1ex]{\linewidth}{1pt}
\begin{proposition}[Division Rule]\label{PRO:w_division}
If $f$, $g$ are  differentiable in the real sense at $c$ and $g(c)\not=0$, then
\begin{align}
\frac{\partial(\frac{f}{g})}{\partial z}(c) &= \frac{\frac{\partial f}{\partial z}(c) g(c) - f(c)\frac{\partial g}{\partial z}(c)}{g^2(c)},\\
\frac{\partial(\frac{f}{g})}{\partial z}(c) &= \frac{\frac{\partial f}{\partial z^*}(c) g(c) - f(c)\frac{\partial g}{\partial z^*}(c)}{g^2(c)}.
\end{align}
\end{proposition}
\noindent\rule[1ex]{\linewidth}{1pt}

\begin{proof}
It follows immediately from the multiplication rule and the reciprocal rule $\left(\frac{f(c)}{g(c)}=f(c)\cdot\frac{1}{g(c)}\right)$.
\end{proof}

\noindent\rule[1ex]{\linewidth}{1pt}
\begin{proposition}[Chain Rule]\label{PRO:w_chain}
If $f$ is  differentiable in the real sense at $c$ and $g$ is differentiable in the real sense at $f(c)$, then
\begin{align}
\frac{\partial g\circ f}{\partial z}(c) &= \frac{\partial g}{\partial z}(f(c))\frac{\partial f}{\partial z}(c) + \frac{\partial g}{\partial z^*}(f(c))\frac{\partial f^*}{\partial z}(c),\\
\frac{\partial g\circ f}{\partial z^*}(c) &= \frac{\partial g}{\partial z}(f(c))\frac{\partial f}{\partial z^*}(c) + \frac{\partial g}{\partial z^*}(f(c))\frac{\partial f^*}{\partial z^*}(c).
\end{align}
\end{proposition}
\noindent\rule[1ex]{\linewidth}{1pt}

\begin{proof}
Consider the function $h(z)=g\circ f(z)=u_{g}(u_{f}(x,y), v_{f}(x,y)) + i v_{g}(u_{f}(x,y), v_{f}(x,y))$. Then the partial derivatives of $u_{h}$ and $v_{h}$ are given by the chain rule:
\begin{align*}
\frac{\partial u_{h}}{\partial x}(c) &= \frac{\partial u_{g}}{\partial x}(f(c))\frac{\partial u_{f}}{\partial x}(c) + \frac{\partial u_{g}}{\partial y}(f(c))\frac{\partial v_{f}}{\partial x}(c),\\
\frac{\partial u_{h}}{\partial y}(c) &= \frac{\partial u_{g}}{\partial x}(f(c))\frac{\partial u_{f}}{\partial y}(c) + \frac{\partial u_{g}}{\partial y}(f(c))\frac{\partial v_{f}}{\partial y}(c),\\
\frac{\partial v_{h}}{\partial x}(c) &= \frac{\partial v_{g}}{\partial x}(f(c))\frac{\partial u_{f}}{\partial x}(c) + \frac{\partial v_{g}}{\partial y}(f(c))\frac{\partial v_{f}}{\partial x}(c),\\
\frac{\partial v_{h}}{\partial y}(c) &= \frac{\partial v_{g}}{\partial x}(f(c))\frac{\partial u_{f}}{\partial y}(c) + \frac{\partial v_{g}}{\partial y}(f(c))\frac{\partial v_{f}}{\partial y}(c).
\end{align*}
In addition, we have:
\begin{align*}
\frac{\partial g}{\partial z}(f(c))\frac{\partial f}{\partial z}(c) =&
\frac{1}{4}\left(
\frac{\partial u_{g}}{\partial x}(f(c)) \frac{\partial u_{f}}{\partial x}(c)
+ \frac{\partial u_{g}}{\partial x}(f(c)) \frac{\partial v_{f}}{\partial y}(c)
+ i\frac{\partial u_{g}}{\partial x}(f(c)) \frac{\partial v_{f}}{\partial x}(c)
- i\frac{\partial u_{g}}{\partial x}(f(c)) \frac{\partial u_{f}}{\partial y}(c)
\right.\\
& +\frac{\partial v_{g}}{\partial y}(f(c)) \frac{\partial u_{f}}{\partial x}(c)
+ \frac{\partial v_{g}}{\partial y}(f(c)) \frac{\partial v_{f}}{\partial y}(c)
+ i\frac{\partial v_{g}}{\partial y}(f(c)) \frac{\partial v_{f}}{\partial x}(c)
- i\frac{\partial v_{g}}{\partial y}(f(c)) \frac{\partial u_{f}}{\partial y}(c)\\
& +i \frac{\partial v_{g}}{\partial x}(f(c)) \frac{\partial u_{f}}{\partial x}(c)
+ i\frac{\partial v_{g}}{\partial x}(f(c)) \frac{\partial v_{f}}{\partial y}(c)
- \frac{\partial v_{g}}{\partial x}(f(c)) \frac{\partial v_{f}}{\partial x}(c)
+ \frac{\partial v_{g}}{\partial x}(f(c)) \frac{\partial u_{f}}{\partial y}(c)\\
& \left.
- i\frac{\partial u_{g}}{\partial y}(f(c)) \frac{\partial u_{f}}{\partial x}(c)
- i\frac{\partial u_{g}}{\partial y}(f(c)) \frac{\partial v_{f}}{\partial y}(c)
+  \frac{\partial u_{g}}{\partial y}(f(c)) \frac{\partial v_{f}}{\partial x}(c)
-  \frac{\partial u_{g}}{\partial y}(f(c)) \frac{\partial u_{f}}{\partial y}(c)\right)
\end{align*}
and
\begin{align*}
\frac{\partial g}{\partial z^*}(f(c))\frac{\partial f^*}{\partial z}(c) =&
\frac{1}{4}\left(
  \frac{\partial u_{g}}{\partial x}(f(c)) \frac{\partial u_{f}}{\partial x}(c)
- \frac{\partial u_{g}}{\partial x}(f(c)) \frac{\partial v_{f}}{\partial y}(c)
- i\frac{\partial u_{g}}{\partial x}(f(c)) \frac{\partial v_{f}}{\partial x}(c)
- i\frac{\partial u_{g}}{\partial x}(f(c)) \frac{\partial u_{f}}{\partial y}(c)
\right.\\
&
-  \frac{\partial v_{g}}{\partial y}(f(c)) \frac{\partial u_{f}}{\partial x}(c)
+  \frac{\partial v_{g}}{\partial y}(f(c)) \frac{\partial v_{f}}{\partial y}(c)
+ i\frac{\partial v_{g}}{\partial y}(f(c)) \frac{\partial v_{f}}{\partial x}(c)
+ i\frac{\partial v_{g}}{\partial y}(f(c)) \frac{\partial u_{f}}{\partial y}(c)\\
&
+ i\frac{\partial v_{g}}{\partial x}(f(c)) \frac{\partial u_{f}}{\partial x}(c)
- i\frac{\partial v_{g}}{\partial x}(f(c)) \frac{\partial v_{f}}{\partial y}(c)
+ \frac{\partial v_{g}}{\partial x}(f(c)) \frac{\partial v_{f}}{\partial x}(c)
+ \frac{\partial v_{g}}{\partial x}(f(c)) \frac{\partial u_{f}}{\partial y}(c)\\
&
\left.+ i\frac{\partial u_{g}}{\partial y}(f(c)) \frac{\partial u_{f}}{\partial x}(c)
- i\frac{\partial u_{g}}{\partial y}(f(c)) \frac{\partial v_{f}}{\partial y}(c)
+  \frac{\partial u_{g}}{\partial y}(f(c)) \frac{\partial v_{f}}{\partial x}(c)
+  \frac{\partial u_{g}}{\partial y}(f(c)) \frac{\partial u_{f}}{\partial y}(c)\right).
\end{align*}
Summing up the last two relations and eliminating the opposite terms, we obtain:
\begin{align*}
\frac{\partial g}{\partial z}(f(c))\frac{\partial f}{\partial z}(c) +\frac{\partial g}{\partial z^*}(f(c))\frac{\partial f^*}{\partial z}(c) =& \frac{1}{2}\left(
\frac{\partial u_{g}}{\partial x}(f(c)) \frac{\partial u_{f}}{\partial x}(c)
- i\frac{\partial u_{g}}{\partial x}(f(c)) \frac{\partial u_{f}}{\partial y}(c)
+ \frac{\partial u_{g}}{\partial y}(f(c)) \frac{\partial u_{f}}{\partial y}(c)  \right.\\
& + i\frac{\partial u_{g}}{\partial y}(f(c)) \frac{\partial v_{f}}{\partial x}(c)
+ i\frac{\partial v_{g}}{\partial x}(f(c)) \frac{\partial u_{f}}{\partial x}(c)
+ \frac{\partial v_{g}}{\partial x}(f(c)) \frac{\partial u_{f}}{\partial y}(c)\\
& \left. -i \frac{\partial u_{g}}{\partial y}(f(c)) \frac{\partial v_{f}}{\partial y}(c)
+ \frac{\partial u_{g}}{\partial y}(f(c)) \frac{\partial v_{f}}{\partial x}(c)\right)\\
=& \frac{1}{2}\left( \frac{\partial u_{h}}{\partial x}(c) + \frac{\partial v_{h}}{\partial y}(c)\right)
+ \frac{i}{2}\left( \frac{\partial v_{h}}{\partial x}(c) - \frac{\partial u_{h}}{\partial y}(c)\right)\\
=& \frac{\partial h}{\partial z}(c).
\end{align*}

To prove the chain rule of the CW-derivative, we apply the chain rule of the W-derivative as well as Propositions \ref{PRO:w_rule1}, \ref{PRO:w_rule2} and obtain:
\begin{align*}
\frac{\partial h}{\partial z^*}(c) =& \left(\frac{\partial h^*}{\partial z}\right)^* = \left( \frac{\partial g^*}{\partial z}(f(c)) \frac{\partial f}{\partial z}(c) + \frac{\partial g^*}{\partial z^*}(f(c))\frac{\partial f^*}{\partial z}(c) \right)^*\\
=& \left( \frac{\partial g^*}{\partial z}(f(c))\right)^* \left(\frac{\partial f}{\partial z}(c)\right)^* + \left(\frac{\partial g^*}{\partial z^*}(f(c))\right)^*\left(\frac{\partial f^*}{\partial z}(c) \right)^*\\
=& \frac{\partial g}{\partial z^*}(f(c)) \frac{\partial f^*}{\partial z^*}(c) + \frac{\partial g}{\partial z}(f(c))\frac{\partial f}{\partial z^*}(c),
\end{align*}
which completes the proof.
\end{proof}

In the following we examine some examples in order to make the aforementioned rules more intelligible.\vspace{1em}
\begin{enumerate}
\item $f(z)=z^2$. Since $f$ is complex analytic, $\frac{\partial f}{\partial z^*}(z)=0$ and $\frac{\partial f}{\partial z}(z)=2z$, for all $z\in\C$.
    \vspace{1em}
\item $f(z)=z^*$. Since $f$ is conjugate-complex analytic, $\frac{\partial f}{\partial z}(z)=0$ and $\frac{\partial f}{\partial z^*}(z)=1$, for all $z\in\C$.
    \vspace{1em}
\item $f(z)=z^3 - iz + (z^*)^2$. Applying the linearity rule we take: $\frac{\partial f}{\partial z}(z)=\frac{\partial (z^3)}{\partial z} + \frac{\partial (iz)}{\partial z}+ \frac{\partial ((z^*)^2)}{\partial z}= 3z^2 + i$. Similarly, $\frac{\partial f}{\partial z^*}(z)=\frac{\partial (z^3)}{\partial z^*} + \frac{\partial (iz)}{\partial z^*}+ \frac{\partial ((z^*)^2)}{\partial z^*}= 2z^*$.
    \vspace{1em}
\item $f(z)=\frac{1}{z}$. Applying the reciprocal rule: $\frac{\partial f}{\partial z}(z) = -\frac{1}{z^2}$, $\frac{\partial f}{\partial z^*}(z) = 0$, for all $z\in\C-\{0\}$.
    \vspace{1em}
\item $f(z)=\left(z^2+z^*\right)^3$. Consider the functions $g(z)=z^3$ and $h(z)=z^2+z^*$. Then $f(z)=g\circ h(z)$. Applying the chain rule: $\frac{\partial f}{\partial z}(z) = \frac{\partial g}{\partial z}((z^2+z^*)^3) \frac{\partial h}{\partial z}(z) + \frac{\partial g}{\partial z^*}((z^2+z^*)^3) \frac{\partial h^*}{\partial z}(z) = 3(z^2+z^*)^2 2z + 0 = 6z(z^2+z^*)^2$. Similarly,
    $\frac{\partial f}{\partial z^*}(z) = 3(z^2+z^*)^2$.
    \vspace{1em}
\end{enumerate}

It should have become clear by now, that all the aforementioned rules can be summarized to the following simple statements:
\textit{
\begin{quote}
\begin{itemize}
\item To compute the W-derivative of a function $f$, which is expressed in terms of $z$ and $z^*$, apply the known differentiation rules considering $z^*$ as a constant.
\item To compute the CW-derivative of a function $f$, which is expressed in terms of $z$ and $z^*$, apply the known differentiation rules considering $z$ as a constant.
\end{itemize}
\end{quote}
}
Most texts or papers, that deal with Wirtinger's methodology, highlight the aforementioned rules only, disregarding the underlying rich mathematical body. Therefore, the first-time reader is left puzzled and with many unanswered questions. For example, it is difficult for the beginner to comprehend the notion ``keep $z^*$ constant'', since if $z^*$ is kept fixed, then so does $z$. We should emphasize, that these statements should be regarded as a simple computational trick, rather than as a rigorous mathematical rule. This trick works well, as shown in the examples, due to Theorem \ref{THE:wirti}, which states the W and CW derivative vanish for conjugate-complex analytic and complex analytic functions respectively, and the differentiation rules (Propositions \ref{PRO:w_linearity}-\ref{PRO:w_chain}). Nonetheless, special care should be considered, whenever this trick is applied. Given the function $f(z)=|z|^2$, one might conclude that $\frac{\partial f}{\partial z^*}=0$, since if we consider $z$ as a constant, according to the aforementioned statements, then $f(z)$ is also a constant. However, on a closer look, one may recast $f$ as $f(z)=zz^*$. Then the same trick produces $\frac{\partial f}{\partial z^*}=z$. Which one is correct?

The answer, of course, is that $\frac{\partial f}{\partial z^*}=z$. One should not conclude that $\frac{\partial f}{\partial z^*}=0$, since $f$ is not conjugate-holomorphic. We can only apply the aforementioned tricks on functions $f$ that are expressed in terms of $z$ and $z^*$ (hence the term $f(z,z^*)$ used often in the literature) in such a way, that: a) if we replace $z$ with a fixed $w$, while keeping $z^*$ intact, the resulting function will be conjugate-complex analytic and b) if we replace $z^*$ with a fixed $w$, while keeping $z$ intact, the resulting function will be complex analytic. Utilizing the methodology employed in Proposition \ref{PRO:CR2}, one may prove that this is possible for any  real analytic function $f$. Of course, another course of action (a more formal one) is to disregard the aforementioned tricks and use directly the differentiation rules (as seen in the examples).

\subsection{Second Order Taylor's expansion formula}\label{SEC:Taylor_sec_ord}
Let $f$ a complex function differentiable in the real sense. Following the same procedure that led us to \ref{EQ:wirti2} we may derive Taylor's expansion formulas of any order. In this section we consider the second order expansion, since it is useful in many problems that employ Newton-based minimization. Consider the second order Taylor's expansion of $u$ and $v$ around $c=c_1+ic_2$:
\begin{align*}
u(c+h) = u(c_1+h_1,c_2+h_2) =& u(c_1,c_2) + \frac{\partial u}{\partial x}(c_1,c_2) h_1 + \frac{\partial u}{\partial y}(c_1,c_2) h_2 \\
& + \frac{1}{2} (h_1, h_2)\cdot H_u\cdot \left(h_1,h_2\right)^T + o(|h|^2), \\
v(c+h) = v(c_1+h_1,c_2+h_2) =& v(c_1,c_2) + \frac{\partial v}{\partial x}(c_1,c_2) h_1 + \frac{\partial v}{\partial y}(c_1,c_2) h_2 \\
& + \frac{1}{2} (h_1, h_2)\cdot H_v\cdot \left(h_1,h_2\right)^T + o(|h|^2),
\end{align*}
where $H_u$, $H_v$ are the associated Hessian matrices. Since $f(c+h) = u(c+h) + iv(c+h)$, after some algebra one obtains that:
\begin{align*}
f(c+h) = f(c) + \left(\frac{\partial f}{\partial z}, \frac{\partial f}{\partial z^*} \right)\cdot \left(\begin{matrix} h \cr h^*\end{matrix}\right)
+ \frac{1}{2}(h, h^*)\cdot\left(
\begin{matrix}\frac{\partial^2 f}{\partial z^2}(c) & \frac{\partial^2 f}{\partial z\partial z^*}(c)\cr
\frac{\partial^2 f}{\partial z^*\partial z}(c) & \frac{\partial^2 f}{\partial (z^*)^2}(c)\end{matrix}\right)\cdot
\left(\begin{matrix}h \cr h^*\end{matrix}\right) + o(|h|^2).
\end{align*}

\subsection{Wirtinger's calculus applied on real valued functions}\label{SEC:wirti_on_cost}

Many problems of complex signal processing involve minimization problems of real valued cost functions defined on complex domains, i.e., $f:X\subset\C\rightarrow\R$. Therefore, in order to successfully implement the associated minimization algorithms, the gradients of the respective cost functions need to be deployed.  Since, such functions are neither complex analytic nor conjugate-complex analytic\footnote{Of course, one may not use a complex valued function as a cost function of a minimization problem, since we cannot define inequalities in $\C$.}, our only option is to compute the gradients, either by employing ordinary $\R^2$ calculus, that is regarding $\C$ as $\R^2$ and evaluating the gradient (i.e., the partial derivatives) of the cost function $f(x+iy)=f(x,y)$, or by using Wirtinger's calculus. Both cases will eventually lead to the same results, but the application of Wirtinger's calculus provides a more elegant and comfortable alternative, especially if the cost function by its definition is given in terms of $z$ and $z^*$ instead of $x$ and $y$ (i.e., the real and imaginary part of $z$).

As the function under consideration $f(z)$ is real valued, the W and CW derivatives are simplified, i.e.,
\begin{align*}
\frac{\partial f}{\partial z}(c) = \frac{1}{2}\left(\frac{\partial f}{\partial x}(c) - i\frac{\partial f}{\partial y}(c)\right) \textrm{ and } \frac{\partial f}{\partial z^*}(c) = \frac{1}{2}\left(\frac{\partial f}{\partial x}(c) + i\frac{\partial f}{\partial y}(c)\right)
\end{align*}
and the following important property can be derived.

\noindent\rule[1ex]{\linewidth}{1pt}
\begin{lemma}\label{LEM:real_function}
If $f:X\subseteq\C\rightarrow\R$ is differentiable in the real sense, then
\begin{align}
\left(\frac{\partial f}{\partial z}(c)\right)^* &= \frac{\partial f}{\partial z^*}(c).
\end{align}
\end{lemma}
\noindent\rule[1ex]{\linewidth}{1pt}

An important consequence is that if $f$ is a real valued function defined on $\C$, then its first order Taylor's expansion at $z$ is given by:
\begin{align*}
f(c+h) & =  f(c) + \frac{\partial f}{\partial z}(c)h + \frac{\partial f}{\partial z^*}(c)h^* +o(|h|)\\
& = f(c) + \frac{\partial f}{\partial z}(c)h + \left(\frac{\partial f}{\partial z}(c)h\right)^* +o(|h|)\\
& = f(c) + \Re\left[\frac{\partial f}{\partial z}(c)h\right] +o(|h|).
\end{align*}
However, in view of the Cauchy Riemann inequality we have:
\begin{align*}
\Re\left[\frac{\partial f}{\partial z}(c)h\right] = \Re\left[\left\langle h, \left(\frac{\partial f}{\partial z}(c)\right)^*\right\rangle_{\C}\right] &\leq \left|\left\langle h, \left(\frac{\partial f}{\partial z}(c)\right)^*\right\rangle_{\C}\right|\\
&\leq |h| \left|\frac{\partial f}{\partial z^*}(c)\right|.
\end{align*}
The equality in the above relationship holds, if $h\upuparrows \frac{\partial f}{\partial z^*}$. Hence, the direction of increase of $f$ is $\frac{\partial f}{\partial z^*}$. Therefore, any gradient descent based algorithm minimizing $f(z)$ is based on the update scheme:
\begin{align}
z_{n} = z_{n-1} - \mu\cdot\frac{\partial f}{\partial z^*}(z_{n-1}).
\end{align}

Assuming differentiability of $f$, a standard result from elementary real calculus states that a necessary condition for a point $(x_0,y_0)$ to be an optimum (in the sense that $f(x,y)$ is minimized) is that this point is a stationary point of $f$, i.e. the partial derivatives of $f$ at $(x_0, y_0)$ vanish. In the context of Wirtinger's calculus we have the following obvious corresponding result.

\noindent\rule[1ex]{\linewidth}{1pt}
\begin{proposition}\label{PRO:first_order_opt}
If $f:X\subseteq\C\rightarrow\R$ is differentiable at $x$ in the real sense, then a necessary condition for a point $c$ to be a local optimum (in the sense that $f(c)$ is minimized or maximized) is that either the W, or the CW derivative vanishes\footnote{Note, that for real valued functions the W and the CW derivatives constitute a conjugate pair (lemma \ref{LEM:real_function}). Thus if the W derivative vanishes, then the CW derivative vanishes too. The converse is also true.}.
\end{proposition}
\noindent\rule[1ex]{\linewidth}{1pt}

\section{Wirtinger's Calculus on general complex Hilbert spaces}\label{SEC:wirti_hilbert}
In this section we generalize the main ideas and results of Wirtinger's calculus on general Hilbert spaces. To this end, we begin with a brief review of the \textit{\frechet derivative}, which generalizes differentiability to abstract Banach spaces.

\subsection{\frechet Derivatives}\label{EQ:frechet}
Consider a Hilbert space $H$ over the field $F$  (typically $\R$ or $\C$). The operator $\bT:H\rightarrow F^\nu$ is said to be \textit{\fredif} at $f_0$, if there exists a linear continuous operator $\bW=(W_1,W_2,\dots,W_\nu)^T:H\rightarrow\F^\nu$ such that
\begin{align}\label{EQ:frechet1}
\lim_{\|h\|_{H}\rightarrow 0}\frac{\left\|\bT(f_0+h)-\bT(f_0)- \bW(h)\right\|_{F^\nu}}{\|h\|_{H}}=0,
\end{align}
where $\|\cdot\|_H=\sqrt{\langle\cdot, \cdot\rangle_H}$ is the induced norm of the corresponding Hilbert Space. Note that $F^\nu$ is considered as a Banach space under the Euclidean norm.  The linear operator $\bW$ is called the \textit{\frechet derivative} and is usually denoted by $d\bT(f_0):H\rightarrow F^\nu$.  Observe that this definition is valid not only for Hilbert spaces, but for general Banach spaces too. However, since we are mainly interested at Hilbert spaces, we present the main ideas in that context. It can be proved that if such a linear continuous operator $\bW$ can be found, then it is unique (i.e., the derivative is unique).
In the special case where $\nu=1$ (i.e., the operator $\bT$ takes values on $F$) using the \textit{Riesz's representation} theorem,  we may replace the linear continuous operator $\bW$ with an inner product. Therefore, the operator $T:H\rightarrow F$ is said to be \textit{\fredif} at $f_0$, iff there exists a $w\in H$, such that
\begin{align}\label{EQ:frechet2}
\lim_{\|h\|_{H}\rightarrow 0}\frac{T(f_0+h)-T(f_0)-\langle h, w\rangle_{H}}{\|h\|_{H}}=0,
\end{align}
where $\langle\cdot, \cdot\rangle_{H}$ is the dot product of the Hilbert space $H$ and $\|\cdot\|_H$ is the
induced norm. The element $w^*$ is usually called the gradient of $T$ at $f_0$ and it is denoted by $w^*=\nabla T(f_0)$.

For a general vector valued operator $\bT=(T_1,\dots,T_\nu)^T:H\rightarrow F^\nu$, we may easily  derive that iff $\bT$ is differentiable at $f_0$, then $T_\iota$ is differentiable at $f_0$, for all $\iota=1,2,\dots,\nu$, and that
\begin{align}\label{EQ:vector_grad}
d\bT(f_0)(h) = \left(\begin{matrix}\langle h, \nabla T_1(f_0)^*\rangle_{H}\cr
\vdots
\cr \langle h, \nabla T_\nu(f_0)^*\rangle_{H}\end{matrix}\right).
\end{align}
To prove this claim, consider that since $\bT$ is differentiable, there exists a continuous linear operator $\bW$ such that
\begin{align*}
\lim_{\|h\|_{H}\rightarrow 0}\frac{\left\|\bT(f_0+h)-\bT(f_0)- \bW(h)\right\|_{F^\nu}}{\|h\|_{H}}=0\Leftrightarrow\\
\lim_{\|h\|_{H}\rightarrow 0} \left( \sum_{\iota=1}^{\nu} \frac{\left|T_\iota(f_0+h)-T_\iota(f_0)- W_\iota(h)\right|_{F}^2}{\|h\|^2_{H}}  \right)  = 0,
\end{align*}
for all $\iota=1,\dots,\nu$. Thus,
\begin{align*}
\lim_{\|h\|_{H}\rightarrow 0} \left( \frac{T_\iota(f_0+h)-T_\iota(f_0)- W_\iota(h)}{\|h\|_{H}}  \right) = 0,
\end{align*}
for all $\iota=1,2,\nu$. The Riesz's representation theorem dictates that since $W_\iota$ is a continuous linear operator, there exists $w_\iota\in H$, such that $W_\iota(h)=\langle h, w_\iota\rangle_H$, for all $\iota=1,\dots,\nu$. This proves that $T_\iota$ is differentiable at $f_0$ and that $w_\iota^*=\nabla T_\iota(f_0)$, thus equation (\ref{EQ:vector_grad}) holds. The converse is proved similarly.

The notion of \textit{\frechet differentiability} may be extended to include also partial derivatives. Consider the operator $T:H^\mu\rightarrow F$ defined on the Hilbert space $H^\mu$ with corresponding inner product:
\begin{align*}
\langle \bbf, \bg\rangle_{H^\mu} = \sum_{\iota=1}^{\mu} \langle f_\iota, g_\iota\rangle_H,
\end{align*}
where $\bbf=(f_1,f_2,\dots f_\mu)$ $\bg=(g_1,g_2,\dots g_\mu)$. $T(\bbf)$ is said to be \textit{\fredif} at $\bbf_0$ in respect with $f_\iota$, iff there exists a $w\in H$, such that
\begin{align}\label{EQ:frechet3}
\lim_{\|h\|_{H}\rightarrow 0}\frac{T(\bbf_0 + [h]_\iota)-T(\bbf_0)-\langle [h]_\iota, w\rangle_{H}}{\|h\|_{H}}=0,
\end{align}
where $[h]_\iota=(0, 0, \dots, 0, h, 0, \dots, 0)^T$, is the element of $H^\mu$ with zero entries everywhere, except at place $\iota$.
The element $w^*$ is called the gradient of $T$ at $\bbf_0$ in respect with $f_\iota$ and it is denoted by $w^*=\nabla_\iota T(\bbf_0)$. The \frechet partial derivative at $\bbf_0$ in respect with $f_\iota$ is denoted by $\frac{\partial T}{\partial f_\iota}(\bbf_0)$, $\frac{\partial T}{\partial f_\iota}(\bbf_0)(\bh)=\langle \bh, w\rangle_{\HH}$.

It is also possible to define \frechet derivatives of higher order and a corresponding Taylor's series expansion. In this context the $n$-th \frechet derivative of $\bT$ at $\bbf_0$, i.e., $d^{n}\bT(\bbf_0)$, is a multilinear map. If $T$ has \frechet derivatives of any order, it can be expanded as a Taylor series, i.e.,
\begin{align}\label{EQ:frechet_Taylor}
\bT(\bbf_0+\bh) = \sum_{n=0}^{\infty} \frac{1}{n!}d^n\bT(\bbf_0)(\bh, \bh, \dots, \bh).
\end{align}
In relative literature the term $d^n\bT(\bc)(\bh, \bh, \dots, \bh)$ is often replaced by $d^n\bT(\bc)\cdot \bh^n$, which it denotes that the multilinear map $d^n\bT(\bc)$ is applied to $(\bh, \bh, \dots, \bh)$.

\subsection{Complex Hilbert spaces}\label{SEC:Hilbert_complex}

Let $\cH$ be a real Hilbert space with inner product $\langle\cdot,\cdot\rangle_{\cH}$. It is easy to verify that $\cH^2=\cH\times\cH$ is also a real Hilbert space with inner product
\begin{align}\label{EQ:cart_inner}
\langle \bbf,\bg\rangle_{\cH^2}=\langle u_{\bbf}, u_{\bg}\rangle_{\cH} + \langle v_{\bbf}, v_{\bg}\rangle_{\cH},
\end{align}
for $\bbf=(u_{\bbf}, v_{\bbf})^T$, $\bg=(u_{\bg}, v_{\bg})^T$. Our objective is to enrich $\cH^2$ with a complex structure. To this end, we define the space $\HH=\{\bbf=u + i v,\; u,v\in\cH\}$ equipped with the complex inner product:
\begin{align}\label{EQ:complex_inner}
\langle \bbf,\bg\rangle_{\HH}=\langle u_{\bbf}, u_{\bg}\rangle_{\cH} + \langle v_{\bbf}, v_{\bg}\rangle_{\cH} + i\left( \langle v_{\bbf}, u_{\bg}\rangle_{\cH} - \langle u_{\bbf}, v_{\bg}\rangle_{\cH}\right),
\end{align}
for $\bbf=u_{\bbf} + iv_{\bbf}$, $\bg= u_{\bg} +iv_{\bg}$. It may be easily proved that $\HH$ is a complex Hilbert space. In the following, this complex structure will be used to derive derivatives similar to the ones obtained from Wirtinger's calculus on $\C^\nu$.

Consider the function $\bT:\A\subseteq\HH\rightarrow\C$, $\bT(\bbf)=\bT(u_{\bbf}+iv_{\bbf})=T_1(u_{\bbf}, v_{\bbf}) + T_2(u_{\bbf}, v_{\bbf}) i$, where $u_{\bbf}, v_{\bbf}\in \cH$ and $T_1, T_2$ are real valued functions defined on $\cH^2$. Any such function, $\bT$, may be regarded as defined either on a subset of $\HH$ or on a subset of $\cH^2$. Furthermore, $\bT$ may be regarded either as a complex valued function, or as a vector valued function, which takes values in $\R^2$. Therefore, we may equivalently write:
\begin{align*}
\bT(\bbf) = \bT(u_{\bbf}+iv_{\bbf}) = T_1(u_{\bbf}, v_{\bbf}) + T_2(u_{\bbf}, v_{\bbf}) i
= \left( T_1(u_{\bbf}, v_{\bbf}), T_2(u_{\bbf}, v_{\bbf}) \right)^T.
\end{align*}
In the following, we will often change the notation according to the specific problem and consider any element of $\bbf\in\HH$ defined either as $\bbf=u_{\bbf} + i v_{\bbf}\in\HH$, or as $\bbf=(u_{\bbf}, v_{\bbf})^T\in \cH^2$. In a similar manner, any complex number may be regarded as either an element of $\C$, or as an element of $\R^2$.
We say that $\bT$ is \textit{\frechet complex differentiable} at $\bc\in\HH$ if there exists $\bw\in\HH$ such that:
\begin{align*}
\lim_{\|\bh\|_{\HH}\rightarrow 0}\frac{\bT(\bc+\bh)-\bT(\bc)-\langle \bh, \bw\rangle_{\HH}}{\|\bh\|_{\HH}}=0.
\end{align*}
Then $\bw^*$ is called the \textit{complex gradient} of $\bT$ at $\bc$ and it is denoted as $\bw^*=\nabla \bT(\bc)$. The \frechet complex derivative of $\bT$ at $\bc$ is denoted as $d\bT(\bc)(\bh)=\langle \bh, \bw\rangle_{\HH}$.
This definition, although similar with the typical \frechet derivative , exploits the complex structure of $\HH$. More specifically, the complex inner product that appears in the definition forces a great deal of structure on $\bT$. Similarly to the case of simple complex functions, from this simple fact follow all the important strong properties of the complex derivative. For example, it can be proved that if $d\bT(\bc)$ exists, then so does $d^{n}\bT(\bc)$, for $n\in\N$.  If $\bT$ is differentiable at any $\bc\in \A$, $\bT$ is called \textit{\frechet holomorphic} in $\A$, or \textit{\frechet complex analytic} in $\A$, in the sense that it can be expanded as a Taylor series, i.e.,
\begin{align}\label{EQ:fre_complex_Taylor}
\bT(\bc+\bh) = \sum_{n=0}^{\infty} \frac{1}{n!}d^n\bT(\bc)(\bh, \bh, \dots, \bh).
\end{align}
The proof of this statement is out of the scope of this manuscript. The expression ``\textit{$\bT$ is \frechet complex analytic at  $\bc$}'' means that $\bT$ is \frechet complex analytic at a neighborhood around $\bc$. We will say that $\bT$ is \textit{\frechet real analytic}, when both $T_1$ and $T_2$ have a Taylor's series expansion in the real sense.

\subsection{Cauchy-Riemann Conditions}\label{SEC:fre_CR_conditions}

We begin our study, exploring the relations between the complex \frechet derivative and the real \frechet derivatives. In the following we will say that $\bT$ is \textit{\frechet differentiable in the complex sense}, if the complex derivative exists, and that $\bT$ is \textit{\frechet differentiable in the real sense}, if its real \frechet derivative exists (i.e., $\bT$ is regarded as a vector valued operator $\bT:\cH^2\rightarrow\cH^2$).

\begin{lemma}\label{LEM:limit}
Consider the Hilbert space $\HH$ and $a, b\in \HH$. The limit
\begin{align}\label{EQ:lem_hilb}
\lim_{\|h\|_{\HH}\rightarrow 0}\frac{\langle h^*, a \rangle_{\HH} - \langle h, b\rangle_{\HH}}{\|h\|_{\HH}} = 0,
\end{align}
if and only if $a=b=\bZero$.
\end{lemma}

\begin{proof}[Proof]
Evidently, if $a=b=\b0$, this limit exists and is equal to zero. For the converse, consider the case where $h=t + i0$, $t\in \cH$. Then equation (\ref{EQ:lem_hilb}) transforms to $\lim_{\|t\|_{\cH}\rightarrow 0}\frac{\langle t, a-b \rangle_{\cH}}{\|t\|_{\cH}} = 0$, which leads to $a-b=0$. Similarly, if $h=0 + i t$, $t\in \cH$, then $\lim_{\|t\|_{\cH}\rightarrow 0}\frac{\langle t, i(a+b) \rangle_{\cH}}{\|t\|_{\cH}} = 0$ and thus $a+b=0$. We conclude that $a=b=\bZero$.
\end{proof}

\noindent\rule[1ex]{\linewidth}{1pt}
\begin{proposition}
Let $\bT:\A\subset\HH\rightarrow\C$ be an operator such that $\bT(\bbf)=\bT(u_{\bbf} + i v_{\bbf}) = \bT(u_{\bbf}, v_{\bbf}) = T_1(u_{\bbf}, v_{\bbf}) + i T_2(u_{\bbf}, v_{\bbf})$.
If the \frechet complex derivative of $\bT$ at a point $\bc\in\A$ (i.e., $d\bT(\bc):\HH\rightarrow\C$) exists, then $T_1$ and $T_2$ are differentiable at the point $\bc=(c_1,c_1)=c_1 + i c_2$, where $c_1, c_2 \in \cH$. Furthermore,
\begin{align}\label{EQ:fre_cauchy-riemann}
\nabla_1 T_1(c_1,c_2)=\nabla_2 T_2(c_1,c_2) \textrm{ and } \nabla_2 T_1(c_1,c_2)=-\nabla_1 T_2(c_1,c_2).
\end{align}
\end{proposition}
\noindent\rule[1ex]{\linewidth}{1pt}

\begin{proof}[Proof]
Considering the first order Taylor expansion of $\bT$ around $\bc$, we take:
\begin{align}\label{EQ:c_tayl1}
\bT(\bc+\bh)=\bT(\bc) + d\bT(\bc)(\bh) + o(\|h\|_{\HH}) = \bT(\bc) + \langle \bh, \nabla\bT(\bc)^*\rangle_{\HH} + o(\|\bh\|_{\HH}),
\end{align}
where the notation $o$ means that $o(\|\bh\|_{\HH})/\|\bh\|_{\HH}\rightarrow 0$, as $\|\bh\|_{\HH}\rightarrow 0$. Substituting $\nabla\bT(\bc)=a + bi$ and $\bh=h_1+i h_2$, $a,b, h_1, h_2\in\cH$,  we have:
\begin{align*}
\bT(\bc+\bh) =& \bT(\bc) + \langle h_1 + ih_2, a - ib \rangle_{\HH} + o(\|\bh\|_{\HH})\\
=& T_1(c_1,c_2) + \langle h_1, a\rangle_{\HH} - \langle h_2, b\rangle_{\HH} + \Re[o(\|\bh\|_{\HH})] +
i \left(T_2(c_1,c_2) + \langle h_2, a\rangle_{\HH} + \langle h_1, b\rangle_{\HH} + \Im[o(\|\bh\|_{\HH})] \right).
\end{align*}
Therefore,
\begin{align}
T_1(c_1+h_1,c_2+h_2) =&  T_1(c_1,c_2)  + \langle h_1, a\rangle_{\HH} - \langle h_2, b\rangle_{\HH} + \Re[o(\|\bh\|_{\HH})],  \label{EQ:rf_tayl1}\\
T_2(c_1+h_1,c_2+h_2) =&  T_2(c_1,c_2) + \langle h_2, a\rangle_{\HH} + \langle h_1, b\rangle_{\HH} + \Im[o(\|\bh\|_{\HH})]. \label{EQ:rf_tayl2}
\end{align}
Since $o(\|\bh\|_{\HH})/\|\bh\|_{\HH}\rightarrow 0$, we also have
\begin{align*}
\Re[o(|(h_1,h_2)|)]/\|(h_1, h_2)\|_{\cH^2}\rightarrow 0 \textrm{ and } \Im[o(|(h_1,h_2)|)]/\|(h_1,h_2)\|_{\cH^2}\rightarrow 0 \textrm{ as } \bh\rightarrow 0.
\end{align*}
Thus, equations (\ref{EQ:rf_tayl1}-\ref{EQ:rf_tayl2}) are the first order Taylor expansions of $T_1$ and $T_2$ around $(c_1,c_2)$. Hence we deduce that:
\begin{align*}
\nabla_1 T_1(c_1,c_2)=a, \nabla_2 T_1(c_1,c_2)=-b,
\nabla_1 T_2(c_1,c_2)=b, \nabla_2 T_2(c_1,c_2)=a.
\end{align*}
This completes the proof.
\end{proof}

Equations (\ref{EQ:fre_cauchy-riemann}) are  the \textit{Cauchy Riemann conditions} with respect to the \frechet notion of differentiability. Similar to the simple case of complex valued functions, they provide a necessary and sufficient condition, for a complex operator $\bT$ defined on $\HH$ to be differentiable in the complex sense, providing that $\bT$ is differentiable in the real sense. This is explored in the following proposition.

\noindent\rule[1ex]{\linewidth}{1pt}
\begin{proposition}\label{PRO:fre_CR2}
If  the operator $\bT:A\subseteq\HH\rightarrow\C$, $\bT(\bbf)=T_1(\bbf) + i T_2(\bbf)$, where $\bbf=u_{\bbf} + i v_{\bbf}$, is \frechet differentiable in the real sense at a point $(c_1,c_2)\in\cH^2$ and the \frechet Cauchy-Riemann conditions hold:
\begin{align}
\nabla_1 T_1(c_1,c_2)=\nabla_2 T_2(c_1,c_2) \textrm{ and } \nabla_2 T_1(c_1,c_2)=-\nabla_1 T_2(c_1,c_2),
\end{align}
then $\bT$ is differentiable in the complex sense at the point $\bc=(c_1,c_2)=c_1+c_2 i\in\HH$.
\end{proposition}
\noindent\rule[1ex]{\linewidth}{1pt}

\begin{proof}
Consider the first order Taylor expansions of $T_1$ and $T_2$ at $\bc=c_1 + i c_2 = (c_1,c_2)^T$:
\begin{align*}
T_1(\bc+\bh) &= T_1(\bc) + \left\langle h_1, \nabla_1 T_1(\bc)\right\rangle_{\cH} + \left\langle h_2, \nabla_2 T_1(\bc)\right\rangle_{\cH} + o(\|\bh\|_{\cH^2}),\\
T_2(\bc+\bh) &= T_2(\bc) + \left\langle h_1, \nabla_1 T_2(\bc)\right\rangle_{\cH}  + \left\langle h_2, \nabla_2 T_2(\bc)\right\rangle_{\cH} + o(\|\bh\|_{\cH^2}).
\end{align*}
Multiplying the second relation with $i$ and adding it to the first one, we take:
\begin{align*}
\bT(\bc+\bh) &= \bT(\bc) + \left\langle h_1, \nabla_1 T_1(\bc)\right\rangle_{\cH} + \left\langle h_2, \nabla_2 T_1(\bc) \right\rangle_{\cH}
+ i\left\langle h_1, \nabla_1 T_2(\bc)\right\rangle_{\cH} + i\left\langle h_2, \nabla_2 T_2(\bc)\right\rangle_{\cH} + o(\|\bh\|_{\cH^2})\\
&= \bT(\bc) + \left\langle h_1,  \nabla_1 T_1(\bc)  - i \nabla_1 T_2(\bc)  \right\rangle_{\HH}  + \left\langle h_2,  \nabla_2 T_1(\bc)  - i \nabla_2 T_2(\bc) \right\rangle_{\HH} + o(\|\bh\|_{\HH}).
\end{align*}
To simplify the notation we may define
\begin{align*}
\nabla_1 \bT(\bc) = \nabla_1 T_1(\bc)  + i \nabla_1 T_2(\bc)  \textrm{ and }
\nabla_2 \bT(\bc) = \nabla_2 T_1(\bc) + i \nabla_2 T_2(\bc)
\end{align*}
and obtain:
\begin{align*}
\bT(\bc+\bh) &= \bT(\bc) + \left\langle h_1, (\nabla_1\bT(\bc))^* \right\rangle_{\HH}  + \left\langle h_2, (\nabla_2 \bT(\bc))^* \right\rangle_{\HH} + o(\|\bh\|_{\cH^2}).
\end{align*}
Next, we substitute $h_1$ and $h_2$ using the relations $h_1= \frac{\bh+\bh^*}{2}$ and $h_2=\frac{\bh-\bh^*}{2i}$ and use the sesquilinear property of the inner product of $\HH$:
\begin{align}
\bT(\bc+\bh) &= \bT(\bc) + \frac{1}{2}\left\langle \bh, \left(\nabla_1\bT(\bc)\right)^* - \frac{1}{i}\left(\nabla_2\bT(\bc)\right)^*\right\rangle_{\HH} +
\frac{1}{2}\left\langle \bh^*, \left(\nabla_1\bT(\bc)\right)^* + \frac{1}{i} \left(\nabla_2\bT(\bc)\right)^*\right\rangle_{\HH}  + o(\|\bh\|_{\cH^2}) \nonumber\\
    &=\bT(\bc) + \frac{1}{2}\left\langle \bh, \left(\nabla_1\bT(\bc) + \frac{1}{i}\nabla_2\bT(\bc)\right)^*\right\rangle_{\HH} +
\frac{1}{2}\left\langle \bh^*, \left(\nabla_1\bT(\bc) - \frac{1}{i} \nabla_2\bT(\bc)\right)^*\right\rangle_{\HH}  + o(\|\bh\|_{\HH})\nonumber\\
 &=\bT(\bc) + \frac{1}{2}\left\langle \bh, \left(\nabla_1\bT(\bc) -i\nabla_2\bT(\bc)\right)^*\right\rangle_{\HH} +
\frac{1}{2}\left\langle \bh^*, \left(\nabla_1\bT(\bc) +i \nabla_2\bT(\bc)\right)^*\right\rangle_{\HH}  + o(\|\bh\|_{\HH}).\label{EQ:frechet_wirti2}
\end{align}
It will be shown that equation (\ref{EQ:frechet_wirti2}) is essential for the development of Wirtinger's calculus. To complete the proof of the proposition we compute the fraction that appears in the definition of the complex \frechet derivative:
\begin{align*}
\frac{\bT(\bc+\bh)-\bT(\bc) - \langle \bh, \bw\rangle_{\HH}}{\|\bh\|_{\HH}}  =& \frac{\frac{1}{2}\left\langle \bh, \left(\nabla_1\bT(\bc) -i\nabla_2\bT(\bc)\right)^*\right\rangle_{\HH} +
\frac{1}{2}\left\langle \bh^*, \left(\nabla_1\bT(\bc) + i \nabla_2\bT(\bc)\right)^*\right\rangle_{\HH}  - \langle \bh, \bw\rangle_{\HH} }{\|\bh\|_{\HH}}\\
& + \frac{o(\|\bh\|_{\HH})}{\|\bh\|_{\HH}}.
\end{align*}
Recall that, since $o(\|\bh\|_{\HH})/\|\bh\|_{\HH}\rightarrow 0$ as $\|\bh\|_{\HH}\rightarrow 0$, for this limit to exist and vanish, it is necessary that $\nabla_1\bT(\bc) + i \nabla_2\bT(\bc)=0$ and $\bw^*=\nabla_1\bT(\bc)(\bc) - i \nabla_2\bT(\bc)$ (see lemma \ref{LEM:limit}).
However, according to our definition,
\begin{align*}
\nabla_1\bT(\bc)  + i \nabla_2\bT(\bc)= \left(\nabla_1 T_1(\bc) - \nabla_2 T_2(\bc)\right) + i \left(\nabla_1 T_2(\bc) + \nabla_2 T_1(\bc)\right).
\end{align*}
Thus, $\bT$ is differentiable in the \frechet complex sense, iff the Cauchy-Riemann conditions hold. Moreover, in this case:
\begin{align*}
\nabla\bT(\bc) =& \frac{1}{2}\left(\nabla_1\bT(\bc)  -i \nabla_2\bT(\bc)\right)\\
         =& \frac{1}{2}\left(\nabla_1 T_1(\bc) + i\nabla_1 T_2(\bc)\right)
           - \frac{i}{2}\left(\nabla_2 T_1(\bc) + i\nabla_2 T_2(\bc)\right)\\
           =& \frac{1}{2}\left(\nabla_1 T_1(\bc) + \nabla_2 T_2(\bc)\right)
           + \frac{i}{2}\left(\nabla_1 T_2(\bc) - \nabla_2 T_1(\bc)\right)\\
         =& \nabla_1 T_1(\bc) + i\nabla_1 T_2(\bc)\\
         =& \nabla_2 T_2(\bc) - i\nabla_2 T_1(\bc).
\end{align*}
\end{proof}

\subsection{\frechet conjugate-complex Derivative}\label{SEC:fre_conj_compl}

An alternative definition of a complex derivative based on the \frechet notion of the differentiability on the Hilbert space $\HH$ is the following. Consider an operator $\bT:A\subseteq\HH\rightarrow\C$, such that $\bT(\bbf) = \bT(u_{\bbf} + i v_{\bbf}) = \bT(u_{\bbf}, v_{\bbf}) = T_1(u_{\bbf}, v_{\bbf}) + iT_2(u_{\bbf}, v_{\bbf})$, $u_{\bbf},v_{\bbf}\in\cH$. We shall say that $\bT$ is \textit{\frechet conjugate-complex differentiable} (or that it is differentiable in the \textit{\frechet conjugate-complex sense}) at $\bc\in\HH$, if there is a $\bw\in\HH$, such that the limit
\begin{align}
\lim_{\|\bh\|_{\HH}\rightarrow 0}\frac{\bT(\bc+\bh) - \bT(\bc) - \langle \bh^*, \bw\rangle_{\HH}}{\|\bh\|_{\HH}} = 0.
\end{align}
The continuous linear operator $d_*\bT(\bc):\HH\rightarrow\C$, such that $d_*\bT(\bc)(\bh)=\langle \bh^*, \bw\rangle_{\HH}$ is called the \textit{\frechet conjugate-complex derivative} of $\bT$ at $\bc$ and the element $\nabla_*\bT(\bc) = w^*\in\HH$ is called the \textit{\frechet conjugate-complex gradient} of $\bT$ at $\bc$. Following a procedure similar to the one presented in section \ref{SEC:fre_CR_conditions}, we may prove the following

\noindent\rule[1ex]{\linewidth}{1pt}
\begin{proposition}
Let $\bT:A\subset\HH\rightarrow\C$ be an operator, such that $\bT(\bbf)=\bT(u_{\bbf} + i v_{\bbf}) = \bT(u_{\bbf}, v_{\bbf}) = T_1(u_{\bbf}, v_{\bbf}) + i T_2(u_{\bbf}, v_{\bbf})$.
If the \frechet conjugate-complex derivative of $\bT$ at a point $\bc\in A$ (i.e., $d_*\bT(\bc):\HH\rightarrow\C$) exists, then $T_1$ and $T_2$ are differentiable at the point $\bc=(c_1,c_1)=c_1 + i c_2$, where $c_1, c_2 \in \cH$. Furthermore,
\begin{align}\label{EQ:fre_cauchy-riemann}
\nabla_1 T_1(c_1,c_2)=-\nabla_2 T_2(c_1,c_2) \textrm{ and } \nabla_2 T_1(c_1,c_2)=\nabla_1 T_2(c_1,c_2).
\end{align}
These are called the \frechet conjugate Cauchy-Riemann conditions.
\end{proposition}

\begin{proposition}\label{PRO:fre_conj_CR2}
If  the operator $\bT:A\subseteq\HH\rightarrow\C$, $\bT(\bbf)=T_1(\bbf) + i T_2(\bbf)$, where $\bbf=u_{\bbf} + i v_{\bbf}$, is \frechet differentiable in the real sense at a point $(c_1,c_2)\in\cH^2$ and the \frechet conjugate Cauchy-Riemann conditions hold:
\begin{align}
\nabla_1 T_1(c_1,c_2)=-\nabla_2 T_2(c_1,c_2) \textrm{ and } \nabla_2 T_1(c_1,c_2)=\nabla_1 T_2(c_1,c_2),
\end{align}
then $\bT$ is differentiable in the \frechet conjugate-complex sense at the point $\bc=(c_1,c_2)=c_1+c_2 i\in\HH$.
\end{proposition}
\noindent\rule[1ex]{\linewidth}{1pt}

If an operator $\bT$ is differentiable in the \frechet conjugate-complex  sense, at every point of an open set $A$, we will say the $\bT$ is \textit{\frechet conjugate holomorphic} on $A$. It can be shown, that then $\bT$ has a form of a Taylor series expansion around $\bc\in A$, i.e.,
\begin{align}\label{EQ:fre_con_complex_Taylor}
\bT(\bc+\bh) = \sum_{n=0}^{\infty} \frac{1}{n!}d_*^n\bT(\bc)(\bh^*, \bh^*,\dots, \bh^*).
\end{align}
In this case, we will say that $\bT$ is \textit{conjugate complex analytic at $\bc$}. Note, that if $\bT(\bbf)$ is complex analytic at $\bc$, then $\bT(\bbf)^*$ is conjugate -complex analytic at $\bc$.

\subsection{\frechet Wirtinger Derivatives}\label{SEC:fre_wirti_deriv}

It is evident, that if neither the \frechet Cauchy Riemann conditions, nor the \frechet conjugate Cauchy-Riemann conditions are satisfied for an operator $\bT$, then the \frechet complex derivatives cannot be exploited and the function cannot be expressed in terms of $\bh$ or $\bh^*$, as in the case of \frechet complex or conjugate-complex differentiable functions.  Nevertheless, if $\bT$ is \frechet differentiable in the real sense (i.e., $T_1$ and $T_2$ are \frechet differentiable), we may still find a form of Taylor's series expansion. Recall, for example, that in the proof of proposition \ref{PRO:fre_CR2}, we concluded, based on the first order Taylor's series expansion of $T_1$, $T_2$, that (equation (\ref{EQ:frechet_wirti2})):
\begin{align*}
\bT(\bc+\bh) =\bT(\bc) + \frac{1}{2}\left\langle \bh, \left(\nabla_1\bT(\bc) -i\nabla_2\bT(\bc)\right)^*\right\rangle_{\HH} +
\frac{1}{2}\left\langle \bh^*, \left(\nabla_1\bT(\bc) +i \nabla_2\bT(\bc)\right)^*\right\rangle_{\HH}  + o(\|\bh\|_{\HH}).
\end{align*}
One may notice that in the more general case, where $\bT$ is \frechet real-differentiable, it's Taylor's expansion is casted in terms of both $\bh$ and $\bh^*$. This can be generalized for higher order Taylor's expansion formulas following the same rationale. Observe also that, if $\bT$ is \frechet complex, or conjugate-complex differentiable, this relation degenerates (due to the Cauchy Riemann conditions) to the respective Taylor's expansion formula (i.e., (\ref{EQ:fre_complex_Taylor}) or (\ref{EQ:fre_con_complex_Taylor})). In this context, the following definitions come naturally.

We define the \textit{\frechet Wirtinger's gradient} (or \textit{W-gradient} for short) of $\bT$ at $\bc$ as
\begin{align}\label{EQ:fre_wirti_der}
\nabla_{\bbf}\bT(\bc) = \frac{1}{2}\left(\nabla_1\bT(\bc)  -i \nabla_2\bT(\bc)\right)
= \frac{1}{2}\left(\nabla_1 T_1(\bc) + \nabla_2 T_2(\bc)\right)
      + \frac{i}{2}\left(\nabla_1 T_2(\bc) - \nabla_2 T_1(\bc)\right),
\end{align}
and the \textit{\frechet Wirtinger's derivative} (or \textit{$W$-derivative}) as $\frac{\partial \bT}{\partial \bbf}(\bc):\HH\rightarrow\C$, such that $\frac{\partial \bT}{\partial \bbf}(\bc)(\bh)=\langle \bh, \nabla_{\bbf}\bT(\bc)^*\rangle_{\HH}$.
Consequently, the \textit{\frechet conjugate Wirtinger's gradient} (or \textit{CW-gradient} for short) and the \textit{\frechet conjugate Wirtinger's derivative} (or \textit{CW-derivative}) of $\bT$ at $\bc$ are defined by:
\begin{align}\label{EQ:fre_conj_wirti_der}
\nabla_{\bbf^*}\bT(\bc) = \frac{1}{2}\left(\nabla_1\bT(\bc)  +  i \nabla_2\bT(\bc)\right)
= \frac{1}{2}\left(\nabla_1 T_1(\bc) - \nabla_2 T_2(\bc)\right)
      + \frac{i}{2}\left(\nabla_1 T_2(\bc) + \nabla_2 T_1(\bc)\right),
\end{align}
and $\frac{\partial \bT}{\partial \bbf^*}(\bc):\HH\rightarrow\C$, such that $\frac{\partial \bT}{\partial \bbf^*}(\bc)(\bh)=\langle\nabla\bh, \left(\nabla_{\bbf^*}\bT(\bc)\right)^*\rangle_{\HH}$.
Note, that both the W-derivative and the CW-derivative
exist, if $\bT$ is \frechet differentiable in the real sense. In view of these new definitions, equation (\ref{EQ:frechet_wirti2}) may now be recasted as follows
\begin{align}\label{EQ:fre_wirti3}
\bT(\bc+\bh) = \bT(\bc) + \left\langle \bh, \left(\nabla_{\bbf}\bT(\bc)\right)^* \right\rangle_{\HH} +
\left\langle\bh^*,  \left(\nabla_{\bbf^*}\bT(\bc)\right)^*\right\rangle_{\HH} + o(\|\bh\|_{\HH}).
\end{align}

At first glance the definitions the W and CW derivatives seem rather obscure. Although, it is evident that they are defined so that that they are consistent with the Taylor's formula (equation (\ref{EQ:frechet_wirti2})), their computation seems quite difficult. However, this is not the case. We will show that they may be computed quickly using simple differentiation rules. First, observe that if $\bT$ satisfies the \frechet Cauchy Riemann conditions then the W-derivative degenerates to the standard complex derivative. The following theorem establishes the fundamental property of W and CW derivatives. Its proof is rather obvious.

\noindent\rule[1ex]{\linewidth}{1pt}
\begin{theorem}\label{THE:fre_wirti}
If $\bT$ is \frechet complex differentiable at $\bc$, then its W derivative degenerates to the standard \frechet complex derivative, while its CW derivative vanishes, i.e.,
\begin{align*}
\nabla_{\bbf}\bT(\bc)=\nabla\bT(\bc),\quad \nabla_{\bbf^*}\bT(\bc)=0.
\end{align*}
Consequently, if $\bT$ is \frechet conjugate-complex differentiable at $\bc$, then its CW derivative degenerates to the standard \frechet conjugate-complex derivative, while its W derivative vanishes, i.e.,
\begin{align*}
\nabla_{\bbf^*}\bT(\bc)=\nabla_{^*}\bT(\bc),\quad \nabla_{\bbf}\bT(\bc)=0.
\end{align*}
\end{theorem}
\noindent\rule[1ex]{\linewidth}{1pt}

In the sequel, we will develop the main differentiation rules of \frechet Wirtinger's derivatives. Most of the proofs of the following properties are straightforward. Nevertheless, we present them all for completeness.

\noindent\rule[1ex]{\linewidth}{1pt}
\begin{proposition}\label{PRO:fre_w_rule1}
If $\bT$ is  \frechet differentiable in the real sense at $\bc$, then
\begin{align}
\left(\nabla_{\bbf}\bT(\bc)\right)^* &= \nabla_{\bbf^*}\bT^*(\bc).\label{EQ:fre_w_rule1}
\end{align}
\end{proposition}
\noindent\rule[1ex]{\linewidth}{1pt}

\begin{proof}
\begin{align*}
\left(\nabla_{\bbf}\bT(\bc)\right)^* &=  \frac{1}{2}\left(\nabla_1 T_1(\bc) + \nabla_2 T_2(\bc)\right)
      - \frac{i}{2}\left(\nabla_1 T_2(\bc) - \nabla_2 T_1(\bc)\right)\\
&=  \frac{1}{2}\left(\nabla_1 T_1(\bc) - \nabla_2(-T_2)(\bc)\right)
      + \frac{i}{2}\left(\nabla_1(-T_2)(\bc) + \nabla_2 T_1(\bc)\right)\\
&=\left(\nabla_{\bbf^*}\bT^*(\bc)\right).
\end{align*}
\end{proof}

\noindent\rule[1ex]{\linewidth}{1pt}
\begin{proposition}\label{PRO:fre_w_rule2}
If $\bT$ is  \frechet differentiable in the real sense at $\bc$, then
\begin{align}
\left(\nabla_{\bbf^*}\bT(\bc)\right)^* &= \nabla_{\bbf}\bT^*(\bc)).\label{EQ:fre_w_rule2}
\end{align}
\end{proposition}
\noindent\rule[1ex]{\linewidth}{1pt}

\begin{proof}
\begin{align*}
\left(\nabla_{\bbf^*}\bT(\bc)\right)^* &=  \frac{1}{2}\left(\nabla_1 T_1(\bc) - \nabla_2 T_2(\bc)\right)
      - \frac{i}{2}\left(\nabla_1 T_2(\bc) + \nabla_2 T_1(\bc)\right)\\
&=  \frac{1}{2}\left(\nabla_1 T_1(\bc) + \nabla_2(-T_2)(\bc)\right)
      + \frac{i}{2}\left(\nabla_1(-T_2)(\bc) - \nabla_2 T_1(\bc)\right)\\
&=\left(\nabla_{\bbf}\bT^*(\bc)\right).
\end{align*}
\end{proof}

\noindent\rule[1ex]{\linewidth}{1pt}
\begin{proposition}[Linearity]\label{PRO:fre_w_linearity}
If $\bT$, $\bS$ are \frechet differentiable in the real sense at $\bc$ and $\alpha, \beta\in\C$, then
\begin{align}
\nabla_{\bbf}(\alpha \bT + \beta \bS)(\bc) &= \alpha\nabla_{\bbf}\bT(\bc) + \beta\nabla_{\bbf}\bS(\bc),\\
\nabla_{\bbf^*}(\alpha \bT + \beta \bS)(\bc) &= \alpha\nabla_{\bbf^*}\bT(\bc) + \beta\nabla_{\bbf^*}\bS(\bc)
\end{align}
\end{proposition}
\noindent\rule[1ex]{\linewidth}{1pt}

\begin{proof}
Let $\bT(\bbf) = \bT(u_{\bbf}+i v_{\bbf}) = \bT(u_{\bbf},v_{\bbf})  = T_1(u_{\bbf},v_{\bbf}) + i T_2(u_{\bbf},v_{\bbf})$, $\bS(\bbf) = \bS(u_{\bbf}+i v_{\bbf}) = \bS(u_{\bbf},v_{\bbf})  = S_1(u_{\bbf},v_{\bbf}) + i S_2(u_{\bbf},v_{\bbf})$ be two complex functions and $\alpha, \beta\in\C$, such that $\alpha=\alpha_1 + i\alpha_2$, $\beta=\beta_1 + i\beta_2$. Then
\begin{align*}
\bR(\bbf) =& \alpha \bT(\bbf) + \beta \bS(\bbf) = (\alpha_1+ i\alpha_2)(T_1(u_{\bbf}, v_{\bbf}) + i T_2(u_{\bbf}, v_{\bbf})) + (\beta_1+ i\beta_2)(S_1(u_{\bbf}, v_{\bbf}) + i S_2(u_{\bbf}, v_{\bbf}))\\
=& \left(\alpha_1 T_1(u_{\bbf}, v_{\bbf}) - \alpha_2 T_2(u_{\bbf}, v_{\bbf}) + \beta_1 S_1(u_{\bbf}, v_{\bbf})  -\beta_2 S_2(u_{\bbf}, v_{\bbf})  \right)\\
& + i\left(\alpha_1 T_2(u_{\bbf}, v_{\bbf}) +\alpha_2 T_1(u_{\bbf}, v_{\bbf}) + \beta_1 S_2(u_{\bbf}, v_{\bbf}) +\beta_2 S_1(u_{\bbf}, v_{\bbf}) \right).
\end{align*}
Thus, the \frechet W-derivative of $\bR$ will be given by:
\begin{align*}
\nabla_{\bbf}\bR(\bc) =& \frac{1}{2}\left(\nabla_1 R_1(\bc) + \nabla_2 R_2(\bc)\right)
      + \frac{i}{2}\left( \nabla_1 R_2(\bc) -  \nabla_2 R_1(\bc)\right)\\
=& \frac{1}{2}\left( \alpha_1 \nabla_1 T_1(\bc) - \alpha_2 \nabla_1 T_2(\bc) + \beta_1 \nabla_1 S_1(\bc) -\beta_2 \nabla_1 S_2(\bc)\right.\\
    & \left. +\alpha_1 \nabla_2 T_2(\bc)  + \alpha_2 \nabla_2 T_1(\bc) +\beta_1 \nabla_2 S_2(\bc) +\beta_2 \nabla_2 S_1(\bc) \right)\\
    & +\frac{i}{2}\left( \alpha_1 \nabla_1 T_2(\bc) +\alpha_2 \nabla_1 T_1(\bc) + \beta_1 \nabla_1 S_2(\bc) +\beta_2 \nabla_1 S_1(\bc) \right.\\
    & \left. -\alpha_1 \nabla_2 T_1(\bc)  + \alpha_2 \nabla_2 T_2(\bc) -\beta_1 \nabla_2 S_1(\bc) +\beta_2 \nabla_2 S_2(\bc) \right)\\
=& \frac{1}{2}(\alpha_1+i\alpha_2) \nabla_1 T_1(\bc)  + \frac{i}{2}(\alpha_1+i\alpha_2) \nabla_1 T_2(\bc)
        + \frac{1}{2}(\beta_1+i\beta_2) \nabla_1 S_1(\bc) + \frac{i}{2}(\beta_1+i\beta_2) \nabla_1 S_2(\bc)\\
  & + \frac{1}{2}(\alpha_1+i\alpha_2) \nabla_2 T_2(\bc) - \frac{i}{2}(\alpha_1+i\alpha_2)\nabla_2 T_1(\bc)
  +\frac{1}{2}(\beta_1+i\beta_2) \nabla_2 S_2(\bc) - \frac{i}{2}(\beta_1+i\beta_2) \nabla_2 S_1(\bc)\\
=& \alpha\left(\frac{1}{2}\left(\nabla_1 T_1(\bc) + \nabla_2 T_2(\bc)\right) +  \frac{i}{2}\left(\nabla_1 T_2(\bc) - \nabla_2 T_1(\bc)\right)  \right)\\
&+\beta\left(\frac{1}{2}\left(\nabla_1 S_1(\bc) + \nabla_2 S_2(\bc)\right) +  \frac{i}{2}\left(\nabla_1 S_2(\bc) - \nabla_2 S_1(\bc)\right)  \right)\\
=& \alpha \nabla_{\bbf}\bT(\bc) + \beta \nabla_{\bbf}\bS(\bc).
\end{align*}
On the other hand, in view of Propositions \ref{PRO:fre_w_rule2} and \ref{PRO:fre_w_rule1} and the linearity property of the \frechet W-derivative, the \frechet CW-derivative of $\bR$ at $\bc$ will be given by:
\begin{align*}
\nabla_{\bbf^*}\bR(\bc) =& \nabla_{\bbf^*}(\alpha \bT+\beta \bS)(\bc) = \left(\nabla_{\bbf}(\alpha \bT+\beta \bS)^*(\bc)\right)^*\\
    =& \left(\nabla_{\bbf}(\alpha^*\bT^*+\beta^*\bS^*)(\bc)\right)^* = \left(\alpha^*\nabla_{\bbf}\bT^*(\bc) + \beta^*\nabla_{\bbf}\bS^*(\bc) \right)^*\\
    =& \alpha \left(\nabla_{\bbf}\bT^*(\bc)\right)^* + \beta \left(\nabla_{\bbf}\bS^*(\bc)\right)^*
    = \alpha \nabla_{\bbf^*}\bT(\bc) + \beta \nabla_{\bbf^*}\bS(\bc).
\end{align*}
\end{proof}

\noindent\rule[1ex]{\linewidth}{1pt}
\begin{proposition}[Product Rule]\label{PRO:fre_w_product}
If $\bT$, $\bS$ are  \frechet differentiable in the real sense at $\bc\in\HH$, then
\begin{align}
\nabla_{\bbf} (\bT\cdot\bS)(\bc) &= \nabla_{\bbf}\bT(\bc)\bS(\bc) + \bT(\bc)\nabla_{\bbf}\bS(\bc),\\
\nabla_{\bbf^*} (\bT\cdot\bS)(\bc) &= \nabla_{\bbf^*}\bT(\bc)\bS(\bc) + \bT(\bc)\nabla_{\bbf^*}\bS(\bc).
\end{align}
\end{proposition}
\noindent\rule[1ex]{\linewidth}{1pt}

\begin{proof}
Let $\bT(\bbf)=\bT(u_{\bbf}+iv_{\bbf})=\bT(u_{\bbf},v_{\bbf})=T_1(u_{\bbf},v_{\bbf}) + i T_2(u_{\bbf},v_{\bbf})$, $\bS(\bbf)=\bS(u_{\bbf}+iv_{\bbf})=\bS(u_{\bbf},v_{\bbf})=S_1(u_{\bbf},v_{\bbf}) + i S_2(u_{\bbf},v_{\bbf})$,  be two complex functions \frechet differentiable at $\bc$. Consider the complex-valued operator $\bR$ defined as $ \bR(\bbf)=\bT(\bbf)\bS(\bbf)$. Then
\begin{align*}
\bR(\bbf)=(T_1(\bbf)+i T_2(\bbf))(S_1(\bbf)+i S_2(\bbf))=(T_1(\bbf)S_1(\bbf) - T_2(\bbf)S_2(\bbf)) + i(T_1(\bbf)S_2(\bbf) + T_2(\bbf)S_1(\bbf)).
\end{align*}
Hence the \frechet W-derivative of $\bR$ at $\bc$ is given by:
\begin{align*}
\nabla_{\bbf}\bR(\bc) =& \frac{1}{2}\left(\nabla_1 R_1(\bc) + \nabla_2 R_2(\bc) \right) + \frac{i}{2}\left(\nabla_1 R_2(\bc) - \nabla_2 R_1(\bc) \right)\\
=& \frac{1}{2}\left(\nabla_1(T_1S_1 - T_2S_2)(\bc) + \nabla_2(T_1S_2+T_2S_1)(\bc)\right) + \frac{i}{2}\left(\nabla_1(T_1S_2+T_2S_1)(\bc) - \nabla_2(T_1S_1-T_2S_2)(\bc)\right)\\
=& \frac{1}{2}\left(\nabla_1(T_1S_1)(\bc) - \nabla_1(T_2S_2)(\bc)
+ \nabla_2(T_1S_2)(\bc) + \nabla_2(T_2S_1)(\bc)\right) \\
&+ \frac{i}{2}\left(\nabla_1(T_1S_2)(\bc) + \nabla_1(T_2S_1)(\bc)
- \nabla_2(T_1S_1)(\bc) + \nabla_2(T_2S_2)(\bc) \right)
\end{align*}
Applying the chain rule of the ordinary \frechet calculus we take:
\begin{align*}
\nabla_{\bbf}\bR(\bc) =& \frac{1}{2}\left(\nabla_1(T_1)(\bc)S_1(\bc) + \nabla_1 S_1(\bc)T_1(\bc)
- \nabla_1 T_2(\bc)S_2(\bc) - \nabla_1S_2(\bc)T_2(\bc) \right.\\
&\left. + \nabla_2 T_1(\bc)S_2(\bc) + \nabla_2 S_2(\bc)T_1(\bc)
+ \nabla_2 T_2(\bc) S_1(\bc) + \nabla_2 S_1(\bc)T_2(\bc)   \right)  \\
& + \frac{i}{2}\left(\nabla_1 T_1(\bc) S_2(\bc) + \nabla_1 S_2(\bc) T_1(\bc)
+ \nabla_1 T_2(\bc) S_1(\bc) + \nabla_1 S_1(\bc) T_2(\bc)\right.\\
&\left. - \nabla_2 T_1(\bc)S_1(\bc)  - \nabla_2 S_1(\bc) T_1(\bc)
+  \nabla_2 T_2(\bc) S_2(\bc)   +  \nabla_2 S_2(\bc) T_2(\bc) \right).
\end{align*}
After factorization we obtain:
\begin{align*}
\nabla_{\bbf}\bR(\bc)
=& S_1(\bc)\left(\frac{1}{2}\left(\nabla_1 T_1(\bc) + \nabla_2 T_2(\bc)\right) + \frac{i}{2}\left(\nabla_1 T_2(\bc) - \nabla_2 T_2(\bc)\right)\right)\\
&+ S_2(\bc)\left(\frac{1}{2}\left(-\nabla_1 T_2(\bc) + \nabla_2 T_1(\bc)\right) + \frac{i}{2}\left(\nabla_1 T_1(\bc) + \nabla_2 T_2(\bc)\right)\right)\\
&+ T_1(\bc)\left(\frac{1}{2}\left(\nabla_1 S_1(\bc) + \nabla_2 S_2(\bc)\right) + \frac{i}{2}\left(\nabla_1 S_2(\bc) - \nabla_2 S_1(\bc)\right)\right)\\
&+ T_2(\bc)\left(\frac{1}{2}\left(-\nabla_1 S_2(\bc) + \nabla_2 S_1(\bc)\right) + \frac{i}{2}\left(\nabla_1 S_1(\bc) + \nabla_2 S_2(\bc)\right)\right).
\end{align*}
Considering that $1/i=-i$, we take:
\begin{align*}
\nabla_{\bbf}\bR(\bc)
=& S_1(\bc)\left(\frac{1}{2}\left(\nabla_1 T_1(\bc) + \nabla_2 T_2(\bc)\right) + \frac{i}{2}\left(\nabla_1 T_2(\bc) - \nabla_2 T_1(\bc) \right)\right)\\
& + i S_2(\bc)\left(\frac{1}{2}\left(\nabla_1 T_1(\bc) + \nabla_2 T_2(\bc)\right) + \frac{i}{2}\left(\nabla_1 T_2(\bc) - \nabla_2 T_1(\bc)\right)\right)\\
&+ T_1(\bc)\left(\frac{1}{2}\left(\nabla_1 S_1(\bc) + \nabla_2 S_2(\bc)\right) + \frac{i}{2}\left(\nabla_1 S_2(\bc) - \nabla_2 S_1(\bc)\right)\right)\\
&+ i T_2(\bc)\left(\frac{1}{2}\left(\nabla_1 S_1(\bc) + \nabla_2 S_2(\bc)\right) + \frac{i}{2}\left(\nabla_1 S_2(\bc) - \nabla_2 S_1(\bc)\right)\right)\\
=& (S_1(\bc) + i S_2(\bc))\nabla_{\bbf}\bT(\bc) + (T_1(\bc) + i T_2(\bc))\nabla_{\bbf}\bS(\bc),
\end{align*}
which gives the result.

The product rule of the \frechet CW-derivative follows from the product rule of the W-derivative and Propositions \ref{PRO:fre_w_rule1}, \ref{PRO:fre_w_rule1} as follows:
\begin{align*}
\nabla_{\bbf^*}(\bT \bS)(\bc) &= \left(\nabla_{\bbf} (\bT\bS)^*(\bc)\right)^* = \left(\nabla_{\bbf} (\bT^*\bS^*)(\bc)\right)^*\\
& =\left(\nabla_{\bbf}\bT^*(\bc)\bS^*(c) + \nabla_{\bbf} \bS^*(\bc)\bT^*(\bc)\right)^*\\
&= \nabla_{\bbf^*}\bT(\bc)\bS(\bc) + \nabla_{\bbf^*}\bS(\bc)\bT(\bc).
\end{align*}
\end{proof}

\begin{lemma}[Reciprocal Rule]\label{PRO:fre_w_reciprocal}
If $\bT$ is \frechet differentiable in the real sense at $\bc$ and $\bT(\bc)\not=0$, then
\begin{align}
\nabla_{\bbf}\left(\frac{1}{\bT}\right)(\bc) &= - \frac{\nabla_{\bbf}\bT(\bc)}{\bT^2(\bc)},\\
\nabla_{\bbf^*}\left(\frac{1}{\bT}\right)(\bc) &= - \frac{\nabla_{\bbf^*}\bT(\bc)}{\bT^2(\bc)}.
\end{align}
\end{lemma}

\begin{proof}
Let $\bT(\bbf) = \bT(u_{\bbf} + iv_{\bbf}) = \bT(u_{\bbf}, v_{\bbf}) =  T_1(u_{\bbf}, v_{\bbf}) + i T_2(u_{\bbf}, v_{\bbf})$ be a complex function,  \frechet differentiable in the real sense at $\bc$, such that $\bT(\bc)\not=0$. Consider the function $\bR(\bbf)=1/\bT(\bbf)$. Then
\begin{align*}
\bR(\bbf)= \frac{T_1(\bbf)}{T_1^2(\bbf) + T_2^2(\bbf)}  -   i\frac{ T_2(\bbf)}{T_1^2(\bbf) + T_2^2(\bbf)}.
\end{align*}
For the partial derivatives of $R_1$, $R_2$ we have:
\begin{align*}
\nabla_1 R_1(\bc) =& \frac{\nabla_1 T_1(\bc)(T_1^2(\bc) + T_2^2(\bc)) - 2 T_1^2(\bc)\nabla_1 T_1(\bc) - 2T_1(\bc)T_2(\bc)\nabla_1 T_2(\bc)} {(T_1^2(\bc) + T_2^2(\bc))^2},\\
\nabla_2 R_1(\bc) = & \frac{\nabla_2 T_1(\bc)(T_1^2(\bc) + T_2^2(\bc)) - 2T_1^2(\bc)\nabla_2 T_1(\bc) - 2T_1(\bc)T_2(\bc)\nabla_2 T_2(\bc)} {(T_1^2(\bc) + T_2^2(\bc))^2},\\
\nabla_1 R_2(\bc) = & - \frac{\nabla_1 T_2(\bc)(T_1^2(\bc) + T_2^2(\bc)) - 2T_2^2(\bc)\nabla_1 T_2(\bc) - 2T_1(\bc)T_2(\bc)\nabla_1 T_1(\bc)} {(T_1^2(\bc) + T_2^2(\bc))^2},\\
\nabla_2 R_2(\bc) =& - \frac{\nabla_2 T_2(\bc)(T_1^2(\bc) + T_2^2(\bc)) - 2T_2^2(\bc)\nabla_2 T_2(\bc) - 2T_1(\bc)T_2(\bc)\nabla_2 T_1(\bc)} {(T_1^2(\bc) + T_2^2(\bc))^2}.
\end{align*}
Therefore,
\begin{align*}
\nabla_{\bbf}\bR(\bc) = & \frac{1}{2}\left(\nabla_1 R_1(\bc) + \nabla_2 R_2(\bc)\right) + \frac{i}{2}\left( \nabla_1 R_2(\bc) - \nabla_2 R_1(\bc)\right)\\
= & \frac{1}{(T_1^2(\bc) + T_2^2(\bc))^2} \left( \nabla_1 T_1(\bc) \left(-T_1^2(\bc) + T_2^2(\bc) + 2iT_1(\bc)T_2(\bc) \right) \right. \\
 &\left. + \nabla_1 T_2(\bc) \left(-2T_1(\bc)T_2(\bc) - iT_1^2(\bc)+i T_2^2(\bc)\right)   \right.\\
 & \left.  + \nabla_2 T_2(\bc) \left(-T_1^2(\bc)+T_2^2(\bc)+2iT_1(\bc)T_2(\bc)\right)
 + \nabla_2 T_1(\bc) \left(2 T_1(\bc)T_2(\bc) + iT_1^2(\bc) - iT_2^2(\bc)\right)    \right)\\
= &  \frac{T_1^2(\bc) - T_2^2(\bc) - 2iT_1(\bc)T_2(\bc)}{2(T_1^2(\bc) + T_2^2(\bc))^2}
 \left( -  \left( \nabla_1 T_1(\bc) + \nabla_2 T_2(\bc)\right) - i \left( \nabla_1 T_2(\bc) - \nabla_2 T_1(\bc) \right) \right)\\
= & - \frac{\nabla_{\bbf}\bT(\bc)}{\bT^2(\bc)}.
\end{align*}

To prove the corresponding rule of the \frechet CW-derivative we apply the reciprocal rule of the \frechet W-derivative as well as Propositions \ref{PRO:fre_w_rule1}, \ref{PRO:fre_w_rule2}:
\begin{align*}
\nabla_{\bbf^*}\left(\frac{1}{\bT}\right)(\bc) &= \left(\nabla_{\bbf}\left(\frac{1}{\bT^*}\right)(\bc)\right)^*
= \left( - \frac{\nabla_{\bbf} \bT^*(\bc)}{\left(\bT^*(\bc)\right)^2}\right)^*
=  - \frac{\nabla_{\bbf^*}\bT(\bc)}{\left(\bT(\bc)\right)^2}.
\end{align*}
\end{proof}

\noindent\rule[1ex]{\linewidth}{1pt}
\begin{proposition}[Division Rule]\label{PRO:fre_w_division}
If $\bT$, $\bS$ are  \frechet differentiable in the real sense at $\bc$ and $\bS(\bc)\not=0$, then
\begin{align}
\nabla_{\bbf}\left(\frac{\bT}{\bS}\right)(\bc) &= \frac{\nabla_{\bbf}\bT(\bc) \bS(\bc) - \bT(\bc)\nabla_{\bbf}\bS(\bc)}{\bS^2(\bc)},\\
\nabla_{\bbf^*}\left(\frac{\bT}{\bS}\right)(\bc) &= \frac{\nabla_{\bbf^*}\bT(\bc) \bS(\bc) - \bT(\bc)\nabla_{\bbf^*}\bS(\bc)}{\bS^2(\bc)}.
\end{align}
\end{proposition}
\noindent\rule[1ex]{\linewidth}{1pt}

\begin{proof}
It follows immediately from the multiplication rule and the reciprocal rule $\left(\frac{\bT(\bc)}{\bS(\bc)}=\bT(c)\cdot\frac{1}{\bS(\bc)}\right)$.
\end{proof}

\noindent\rule[1ex]{\linewidth}{1pt}
\begin{proposition}[Chain Rule]\label{PRO:fre_w_chain}
Consider the functions $\bT:\HH\rightarrow\C$ and $\bS:\C\rightarrow\C$ so that they are differentiable in the real sense at $\bc$ and $z_0=\bT(\bc)$ respectively. Then the operator $\bR=\bS\circ\bT$ is differentiable in the real sense at $\bc$, and
\begin{align}
\nabla_{\bbf}\bS\circ\bT(\bc) &= \frac{\partial \bS}{\partial z}(\bT(c))\nabla_{\bbf}\bT(\bc) + \frac{\partial \bS}{\partial z^*}(\bT(\bc))\nabla_{\bbf}(\bT^*)(\bc),\\
\nabla_{\bbf^*}\bS\circ\bT(\bc) &= \frac{\partial \bS}{\partial z}(\bT(\bc))\nabla_{\bbf^*}\bT(\bc) + \frac{\partial \bS}{\partial z^*}(\bT(\bc))\nabla_{\bbf^*}(\bT^*)(\bc).
\end{align}
\end{proposition}
\noindent\rule[1ex]{\linewidth}{1pt}

\begin{proof}
Consider the function $\bR(\bbf)=\bS\circ \bT(\bbf)= u_{\bR}(\bbf) +i v_{\bR}(\bbf) = u_{\bS}(T_1(\bbf), T_2(\bbf)) + i v_{\bS}(T_1(\bbf), T_2(\bbf))$. Then the \frechet partial derivatives of $R_1$ and $R_2$ are given by the chain rule:
\begin{align*}
\nabla_1 R_1(\bc) &= \frac{\partial u_{\bS}}{\partial x}(\bT(\bc))\nabla_1 T_1(\bc) + \frac{\partial u_{\bS}}{\partial y}(\bT(\bc))\nabla_1 T_2(\bc),\\
\nabla_2 R_1(\bc) &= \frac{\partial u_{\bS}}{\partial x}(\bT(\bc))\nabla_2 T_1(\bc) + \frac{\partial u_{\bS}}{\partial y}(\bT(\bc))\nabla_2 T_2(\bc),\\
\nabla_1 R_2(\bc) &= \frac{\partial v_{\bS}}{\partial x}(\bT(\bc))\nabla_1 T_1(\bc) + \frac{\partial v_{\bS}}{\partial y}(\bT(\bc))\nabla_1 T_2(\bc),\\
\nabla_2 R_2(\bc) &= \frac{\partial v_{\bS}}{\partial x}(\bT(\bc))\nabla_2 T_1(\bc) + \frac{\partial v_{\bS}}{\partial y}(\bT(\bc))\nabla_2 T_2(\bc).
\end{align*}
In addition, we have:
\begin{align*}
\frac{\partial \bS}{\partial z}(\bT(\bc))\nabla_{\bbf}\bT(\bc) = &
\frac{1}{4}\left(
\frac{\partial u_{\bS}}{\partial x}(\bT(\bc))\nabla_1 T_1(\bc)
+ \frac{\partial u_{\bS}}{\partial x}(\bT(\bc))\nabla_2 T_2(\bc)
+ i\frac{\partial u_{\bS}}{\partial x}(\bT(\bc))\nabla_1 T_2(\bc)
- i\frac{\partial u_{\bS}}{\partial x}(\bT(\bc))\nabla_2 T_1(\bc)
\right.\\
& +\frac{\partial v_{\bS}}{\partial y}(\bT(\bc)) \nabla_1 T_1(\bc)
+ \frac{\partial v_{\bS}}{\partial y}(\bT(\bc)) \nabla_2 T_2(\bc)
+ i\frac{\partial v_{\bS}}{\partial y}(\bT(\bc)) \nabla_1 T_2(\bc)
- i\frac{\partial v_{\bS}}{\partial y}(\bT(\bc)) \nabla_2 T_1(\bc)\\
& + i \frac{\partial v_{\bS}}{\partial x}(\bT(\bc)) \nabla_1 T_1(\bc)
+ i\frac{\partial v_{\bS}}{\partial x}(\bT(\bc)) \nabla_2 T_2(\bc)
- \frac{\partial v_{\bS}}{\partial x}(\bT(\bc)) \nabla_1 T_2(\bc)
+ \frac{\partial v_{\bS}}{\partial x}(\bT(\bc)) \nabla_2 T_1(\bc)\\
& \left.
- i\frac{\partial u_{\bS}}{\partial y}(\bT(\bc)) \nabla_1 T_1(\bc)
- i\frac{\partial u_{\bS}}{\partial y}(\bT(\bc)) \nabla_2 T_2(\bc)
+  \frac{\partial u_{\bS}}{\partial y}(\bT(\bc)) \nabla_1 T_2(\bc)
-  \frac{\partial u_{\bS}}{\partial y}(\bT(\bc)) \nabla_2 T_1(\bc)\right)
\end{align*}
and
\begin{align*}
\frac{\partial \bS}{\partial z^*}(\bT(\bc))\nabla_{\bbf}\bT^*(\bc) = &
\frac{1}{4}\left(
\frac{\partial u_{\bS}}{\partial x}(\bT(\bc))\nabla_1 T_1(\bc)
- \frac{\partial u_{\bS}}{\partial x}(\bT(\bc))\nabla_2 T_2(\bc)
- i\frac{\partial u_{\bS}}{\partial x}(\bT(\bc))\nabla_1 T_2(\bc)
- i\frac{\partial u_{\bS}}{\partial x}(\bT(\bc))\nabla_2 T_1(\bc)
\right.\\
& -\frac{\partial v_{\bS}}{\partial y}(\bT(\bc)) \nabla_1 T_1(\bc)
+ \frac{\partial v_{\bS}}{\partial y}(\bT(\bc)) \nabla_2 T_2(\bc)
+ i\frac{\partial v_{\bS}}{\partial y}(\bT(\bc)) \nabla_1 T_2(\bc)
+ i\frac{\partial v_{\bS}}{\partial y}(\bT(\bc)) \nabla_2 T_1(\bc)\\
& + i \frac{\partial v_{\bS}}{\partial x}(\bT(\bc)) \nabla_1 T_1(\bc)
- i\frac{\partial v_{\bS}}{\partial x}(\bT(\bc)) \nabla_2 T_2(\bc)
+ \frac{\partial v_{\bS}}{\partial x}(\bT(\bc)) \nabla_1 T_2(\bc)
+ \frac{\partial v_{\bS}}{\partial x}(\bT(\bc)) \nabla_2 T_1(\bc)\\
& \left.
+ i\frac{\partial u_{\bS}}{\partial y}(\bT(\bc)) \nabla_1 T_1(\bc)
- i\frac{\partial u_{\bS}}{\partial y}(\bT(\bc)) \nabla_2 T_2(\bc)
+  \frac{\partial u_{\bS}}{\partial y}(\bT(\bc)) \nabla_1 T_2(\bc)
+  \frac{\partial u_{\bS}}{\partial y}(\bT(\bc)) \nabla_2 T_1(\bc)\right).
\end{align*}
Summing up the last two relations and eliminating the opposite terms, we obtain:
\begin{align*}
\frac{\partial \bS}{\partial z}(\bT(\bc))\nabla_{\bbf}\bT(\bc) + \frac{\partial \bS}{\partial z^*}(\bT(\bc))\nabla_{\bbf}\bT^*(\bc) =& \frac{1}{2}\left(
\frac{\partial u_{\bS}}{\partial x}(\bT(\bc)) \nabla_1 T_1(\bc)
- i\frac{\partial u_{\bS}}{\partial x}(\bT(\bc)) \nabla_2 T_1(\bc)
+ \frac{\partial u_{\bS}}{\partial y}(\bT(\bc)) \nabla_2 T_1(\bc)  \right.\\
& + i\frac{\partial u_{\bS}}{\partial y}(\bT(\bc)) \nabla_1 T_2(\bc)
+ i\frac{\partial v_{\bS}}{\partial x}(\bT(\bc)) \nabla_1 T_1(\bc)
+ \frac{\partial v_{\bS}}{\partial x}(\bT(\bc)) \nabla_2 T_1(\bc)\\
& \left. -i \frac{\partial u_{\bS}}{\partial y}(\bT(\bc)) \nabla_2 T_2(\bc)
+ \frac{\partial u_{\bS}}{\partial y}(\bT(\bc)) \nabla_1 T_2(\bc)\right)\\
=& \frac{1}{2}\left( \nabla_1 R_1(\bc) + \nabla_2 R_2(\bc)\right)
+ \frac{i}{2}\left( \nabla_1 R_2(\bc) - \nabla_1 R_1(\bc)\right)\\
=& \nabla_{\bbf}\bR(\bc).
\end{align*}

To prove the chain rule of the \frechet CW-derivative, we apply the chain rule of the W-derivative as well as Propositions \ref{PRO:fre_w_rule1}, \ref{PRO:fre_w_rule2} and obtain:
\begin{align*}
\nabla_{\bbf^*}\bR(\bc) =& \left(\nabla_{\bbf}\bR^*(\bc)\right)^* = \left( \frac{\partial \bS^*}{\partial z}(\bT(\bc))\nabla_{\bbf}\bT(\bc) + \frac{\partial \bS^*}{\partial z^*}(\bT(\bc)) \nabla_{\bbf} \bT^*(\bc) \right)^*\\
=& \left( \frac{\partial \bS^*}{\partial z}(\bT(\bc))\right)^* \left(\nabla_{\bbf}\bT(\bc)\right)^* + \left(\frac{\partial \bS^*}{\partial z^*}(\bT(\bc))\right)^* \left(\nabla_{\bbf} \bT^*(\bc) \right)^*\\
=& \frac{\partial \bS}{\partial z^*}(\bT(\bc)) \nabla_{\bbf^*} \bT^*(\bc)  + \frac{\partial \bS}{\partial z}(\bT(\bc)) \nabla_{\bbf^*} \bT(\bc),
\end{align*}
which completes the proof.
\end{proof}

The following rules may be immediately proved using the definition of \frechet W and CW derivatives and the aforementioned rules.
\begin{enumerate}
\item If $\bT(\bbf)=\langle \bbf, \bw\rangle_\HH$, then $\nabla_{\bbf}\bT=\bw^*$, $\nabla_{\bbf^*}\bT=\bZero$.
\item If $\bT(\bbf)=\langle \bw, \bbf\rangle_\HH$, then $\nabla_{\bbf}\bT=\bZero$, $\nabla_{\bbf^*}\bT=\bw$.
\item If $\bT(\bbf)=\langle \bbf^*, \bw\rangle_\HH$, then $\nabla_{\bbf}\bT=\bZero$, $\nabla_{\bbf^*}\bT=\bw^*$.
\item If $\bT(\bbf)=\langle \bw, \bbf^*\rangle_\HH$, then $\nabla_{\bbf}\bT=\bw$, $\nabla_{\bbf^*}\bT=\bZero$.
\end{enumerate}

\subsection{Generalized Wirtinger's calculus applied on real valued functions}\label{SEC:wirti_on_cost}

In non-linear complex signal processing, we are often interested in minimization problems of real valued cost functions defined on certain complex Hilbert spaces. Therefore, in order to successfully implement the associated minimization algorithms, the gradients of the respective cost functions need to be deployed.  We may compute the gradients, either by employing ordinary \frechet calculus, that is regarding the complex Hilbert space as a cartesian product of real Hilbert spaces, or by using Wirtinger's calculus. Both cases will eventually lead to the same results, but the application of Wirtinger's calculus provides a more elegant and comfortable alternative, especially if the cost function, by its definition, is given in terms of $\bbf$ and $\bbf^*$ (where $\bbf$ is an element of the respective Hilbert space).

As the function under consideration $T(\bbf)$ is real valued, the \frechet W and CW derivatives are simplified, i.e.,
\begin{align*}
\nabla_{\bbf}T(\bc) = \frac{1}{2}\left(\nabla_1 T(\bc) - i\nabla_2 T(\bc)\right) \textrm{ and } \nabla_{\bbf^*} T(\bc) = \frac{1}{2}\left(\nabla_1 T(\bc) + i\nabla_2 T(\bc)\right)
\end{align*}
and the following important property can be derived.

\noindent\rule[1ex]{\linewidth}{1pt}
\begin{lemma}\label{LEM:fre_real_function}
If $f:A\subseteq\HH\rightarrow\R$ is \frechet differentiable in the real sense, then
\begin{align}
\left(\nabla_{\bbf}T(\bc)\right)^* &= \nabla_{\bbf^*} T(\bc).
\end{align}
\end{lemma}
\noindent\rule[1ex]{\linewidth}{1pt}

An important consequence is that if $T$ is a real valued function defined on $\HH$, then its first order Taylor's expansion at $\bc$ is given by:
\begin{align*}
T(\bc+\bh) & =  T(\bc) + \left\langle \bh, \left(\nabla_{\bbf}T(\bc)\right)^*\right\rangle_{\HH} + \left\langle \bh^*, \left(\nabla_{\bbf^*}T(\bc)\right)^*\right\rangle_{\HH} +o(\|\bh\|_{\HH})\\
& = T(\bc) + \left\langle \bh, \left(\nabla_{\bbf}T(\bc)\right)^*\right\rangle_{\HH} + \left(\left\langle \bh, \nabla_{\bbf^*}T(\bc)\right\rangle_{\HH}\right)^* + o(\|\bh\|_{\HH})\\
& = T(\bc) + \Re\left[\left\langle \bh, \left(\nabla_{\bbf}T(\bc)\right)^*\right\rangle_{\HH}\right] +o(\|\bh\|_{\HH}).
\end{align*}
However, in view of the Cauchy Riemann inequality we have:
\begin{align*}
\Re\left[\left\langle \bh, \left(\nabla_{\bbf} T(\bc)\right)^*\right\rangle_{\HH}\right] &\leq \left|\left\langle \bh, \left(\nabla_{\bbf} T(\bc)\right)^*\right\rangle_{\HH}\right|\\
&\leq \|\bh\|_{\HH} \left\|\nabla_{\bbf^*}T(\bc)\right\|_{\HH}.
\end{align*}
The equality in the above relationship holds, if $\bh\upuparrows \nabla_{\bbf^*} T(\bc)$. Hence, the direction of increase of $T$ is $\nabla_{\bbf^*}T(\bc)$. Therefore, any gradient descent based algorithm minimizing $T(\bbf)$ is based on the update scheme:
\begin{align}
\bbf_{n} = \bbf_{n-1} - \mu\cdot\nabla_{\bbf^*}T(\bbf_{n-1}).
\end{align}

Assuming differentiability of $T$, a standard result from \frechet real calculus states that a necessary condition for a point $\bc$ to be an optimum (in the sense that $T(\bbf)$ is minimized) is that this point is a stationary point of $T$, i.e. the partial derivatives of $T$ at $\bc$ vanish. In the context of Wirtinger's calculus we have the following obvious corresponding result.

\noindent\rule[1ex]{\linewidth}{1pt}
\begin{proposition}\label{PRO:fre_first_order_opt}
If $T:X\subseteq\HH\rightarrow\C$ is \frechet differentiable at $\bc$ in the real sense, then a necessary condition for a point $\bc$ to be a local optimum (in the sense that $T(\bc)$ is minimized or maximized) is that either the \frechet W, or the CW derivative vanishes\footnote{Note, that for real valued functions the W and the CW derivatives constitute a conjugate pair (lemma \ref{LEM:fre_real_function}). Thus if the W derivative vanishes, then the CW derivative vanishes too. The converse is also true.}.
\end{proposition}
\noindent\rule[1ex]{\linewidth}{1pt}



%
\bibliographystyle{plain}
\bibliography{refs}

\begin{thebibliography}{10}

\bibitem{Adali10}
T.~Adali and H.~Li.
\newblock {\em Complex-valued adaptive signal processing}.
\newblock Adaptive Signal Processing: Next Generation Solutions, T. Adali and
  S. Haykin, editors. Hoboken, NJ, Wiley, 2010.

\bibitem{Adali08a}
T.~Adali, H.~Li, M.~Novey, and J.F. Cardoso.
\newblock Complex {ICA} using nonlinear functions.
\newblock {\em IEEE Trans. Signal Process.}, 56(9):4536--4544, 2008.

\bibitem{Brandwood}
D.~H. Brandwood.
\newblock A complex gradient operator and its application in adaptive array
  theory.
\newblock {\em IEE proc. H (Microwaves, optics and Antennas)}, 130(1):11--16,
  1983.

\bibitem{CaGePaVe}
A.~S. Cacciapuoti, G.~Gelli, L.~Paura, and F.~Verde.
\newblock Widely linear versus linear blind multiuser detection with
  subspace-based channel estimation: Finite sample-size effects.
\newblock {\em IEEE Trans. Signal Process.}, 57(4):1426--1443, 2009.

\bibitem{VanDeBos}
A.~Van de~Bos.
\newblock Complex gradient and hessian.
\newblock {\em IEE proc. Visual image signal processing}, 141(6):380--382,
  1994.

\bibitem{Delga}
K.~Kreutz-Delgado.
\newblock The complex gradient operator and the
  $\mathbb{C}\mathbb{R}$-calculus.
\newblock
  \url{http://citeseerx.ist.psu.edu/viewdoc/download?doi=10.1.1.86.6515&rep=re%
p1&type=pdf}.

\bibitem{ManGoh}
D.~Mandic and V.~S.~L. Goh.
\newblock {\em Complex Valued Nonlinear Adaptive Filters}.
\newblock Wiley, 2009.

\bibitem{MatPaSte}
D.~Mattera, L.~Paura, and F.~Sterle.
\newblock Widely linear decision-feedback equalizer for time-dispersive linear
  {MIMO} channels.
\newblock {\em IEEE Trans. Signal Process.}, 53(7):2525--2536, 2005.

\bibitem{MerkHatzi}
S.~Merkourakis and T.~Hatziafratis.
\newblock {\em Introduction to Complex Analysis (in greek)}.
\newblock Symmetria, 2005.

\bibitem{Moreno}
J.~Navarro-Moreno.
\newblock {ARMA} prediction of widely linear systems by using the innovations
  algorithm.
\newblock {\em IEEE Trans. Signal Process.}, 56(7):3061--3068, 2008.

\bibitem{Nehari}
Z.~Nehari.
\newblock {\em Introduction to Complex Analysis}.
\newblock Allyn and Bacon, Inc., 1961.

\bibitem{Adali08b}
M.~Novey and T.~Adali.
\newblock On extending the complex fast {ICA} algorithm to noncircular sources.
\newblock {\em IEEE Trans. Signal Process.}, 56(5):2148--2154, 2008.

\bibitem{Picin95}
B.~Picinbono and P.~Chevalier.
\newblock Widely linear estimation with complex data.
\newblock {\em IEEE Trans. Signal Process.}, 43(8):2030--2033, 1995.

\bibitem{Remmert}
R.~Remmert.
\newblock {\em Theory of Complex Functions}.
\newblock Springer-Verlag, 1991.

\bibitem{Wirti}
W.~Wirtinger.
\newblock Zur formalen theorie der functionen von mehr complexen
  ver{\"a}nderlichen.
\newblock {\em Math. Ann.}, 97:357--375, 1927.

\end{thebibliography}


%




\end{document}